\documentclass[11pt]{article}
\usepackage{graphicx} 
\usepackage{multirow} 
\newcommand{\whiteboxTV}{\mathsf{LogitAccessTV}}
\newcommand{\Osamp}{\calO_{\mathrm{samp}}}
\newcommand{\Odist}{\calO_{\mathrm{logit}}}
\newcommand{\Osketch}{\calO_{\mathrm{ndist}}}
\newcommand{\Onoisy}{\calO_{\mathrm{ndist}}}
\newcommand{\softmax}{\operatorname{softmax}}

\usepackage{fullpage} 
\usepackage[english]{babel}
\usepackage[utf8x]{inputenc}
\usepackage[T1]{fontenc}
\usepackage{amsmath}
\usepackage{amssymb}
\usepackage{amsthm}
\usepackage{bbm,bm}
\usepackage{bbold}
\usepackage{parskip}
\usepackage{thm-restate}
\usepackage{booktabs}
\usepackage{array}
\usepackage{algorithm}
\usepackage{algpseudocode}
\usepackage{graphicx}
\usepackage{mathtools}
\usepackage{xcolor}
\usepackage[margin=1in]{geometry}
\usepackage{hyperref}
\hypersetup{
	colorlinks=true,
	linkcolor=blue!70!black,
	citecolor=blue!70!black,
	urlcolor=blue!70!black
}
\usepackage{cleveref}
\newtheorem{theorem}{Theorem}
\newtheorem{lemma}{Lemma}

\newtheorem{definition}{Definition}
\newtheorem{corollary}{Corollary}
\newtheorem{proposition}{Proposition}

\newtheorem{problem}{Problem}

\newtheorem{remark}{Remark}

\newcommand{\defeq}{:=}

\newcommand{\eps}{\varepsilon}

\newcommand{\R}{\mathbb{R}}

\newcommand{\N}{\mathbb{N}}

\newcommand{\half}{\frac{1}{2}}

\newcommand{\E}{\mathbb{E}}

\newcommand{\Var}{\textup{Var}}

\newcommand{\dd}{\textup{d}}

\newcommand{\Par}[1]{\left(#1\right)}
\newcommand{\Brack}[1]{\left[#1\right]}

\newcommand{\Abs}[1]{\left|#1\right|}

\newcommand{\ind}{\mathbb{I}}

\newcommand{\supp}{\mathrm{supp}}

\newcommand{\sig}{\sigma}

\newcommand{\csketch}{\mathsf{CoupledIncrement}}
\newcommand{\tvnoise}{\mathsf{NoisyDistAccessTV}}

\newcommand{\KL}[2]{\mathrm{KL}\left(#1\|#2\right)}
\newcommand{\CS}[2]{\chi^2\left(#1\|#2\right)}
\newcommand{\TV}[1]{\mathrm{TV}\left(#1\right)}

\newcommand{\hp}{\widehat{p}}
\newcommand{\hpi}{\widehat{\pi}}
\newcommand{\hmu}{\widehat{\mu}}
\newcommand{\hZ}{\widehat{Z}}

\newcommand{\Unif}{\mathrm{Unif}}
\newcommand{\Bern}{\mathrm{Bern}}

\newcommand{\simiid}{{\overset{\text{i.i.d.}}{\sim}}}
\newcommand{\var}[1]{\mathrm{\bbV ar}\left[#1\right]}

\newcommand{\Prob}{\mathbb{P}}
\renewcommand{\Pr}{\Prob}

\newcommand{\calA}{{\mathcal{A}}}

\newcommand{\calE}{{\mathcal{E}}}

\newcommand{\calG}{{\mathcal{G}}}

\newcommand{\calO}{{\mathcal{O}}}
\newcommand{\calP}{{\mathcal{P}}}

\newcommand{\calZ}{{\mathcal{Z}}}

\newcommand{\bbE}{{\mathbb{E}}}

\newcommand{\bbV}{{\mathbb{V}}}

\newcommand{\paren}[1]{\left(#1\right)}
\newcommand{\bra}[1]{\left[#1\right]}
\newcommand{\cbra}[1]{\left\{#1\right\}}

\crefname{assumption}{assumption}{assumptions}

\definecolor{burntorange}{rgb}{0.8, 0.33, 0.0}

\usepackage{natbib}
\title{Total Variation Distance Estimation in Autoregressive Models}
\author{Eric Price\thanks{University of Texas at Austin, \texttt{ecprice@cs.utexas.edu}.} \and Kevin Tian\thanks{University of Texas at Austin, \texttt{kjtian@cs.utexas.edu}.} \and Zhiyang Xun\thanks{University of Texas at Austin, \texttt{zxun@cs.utexas.edu}.} \and Yusong Zhu\thanks{University of Texas at Austin, \texttt{zhuys@utexas.edu}.}}

\date{}

\begin{document}

\maketitle
{\let\thefootnote\relax
\footnotetext{Authors are listed in alphabetical order.}
\footnotetext{A preliminary version of this paper was accepted to NeurIPS 2026.}
}
\begin{abstract}
Modern LLM deployments use a number of implementation choices and inference optimizations (e.g., batching, custom kernels, quantization) on top of fixed weights, so two engines serving "the same model" can produce meaningfully different distributions.  We study the problem of estimating the total variation (TV) distance between two length-$n$ autoregressive distributions to additive error $\eps$, under three access models:
\begin{itemize}
\item Under \emph{sample} access, we use $\widetilde{O}(\frac{n^2 K}{\eps^2})$ queries, where $K$ is the maximum support of the next-token distribution. This improves upon the $\widetilde{O}(\frac{n^3 m}{\eps^5})$-query estimator of \cite{meel2025distance}, where $m \ge K$ is the total size of the token alphabet.
\item Under \emph{logit} access, we use $O(\frac{n}{\eps^2})$ queries, and this is tight.
\item Under \emph{noisy logit} access, we can smoothly interpolate between the above two guarantees: if probability values are given to $\sigma$ relative error, we use $\widetilde{O}(\frac{n + n^2\sigma^2}{\eps^2})$ queries.
\end{itemize}
We complement our theoretical results with an empirical evaluation of our algorithms, for example measuring the distance between \texttt{sglang} and \texttt{vllm} serving identical weights. 
Our experiments highlight the robustness and practicality of estimating the total variation distance, which remains estimable where the KL divergence is infinite. Our code is available at
\href{https://github.com/XunZhiyang/llm-tv-estimation}
{\texttt{https://github.com/XunZhiyang/llm-tv-estimation}}.
\end{abstract}
\section{Introduction}

While the behavior of a large language model (LLM) is often idealized as being determined by its model weights, in practice a wide variety of choices --- GPUs, batching, kernels, and low-level numerical effects --- can noticeably change its output behavior~\cite{he2025nondeterminism}.  Other inference-time optimizations, e.g., weight quantization \cite{frantar2022gptq, lin2024awq}, KV quantization \cite{LiuYJZXBC024, HooperKMMSKG24}, and lossy speculative decoding methods~\cite{ZhouLRMRKKA24, LiWZZ24}, deliberately introduce approximation errors in exchange for greater efficiency. Other times, it has been documented that model providers simply serve with bugs~\cite{ChuTLI25}. More broadly, different providers of the same open-weight models can exhibit significant differences in latency, throughput, cost, and performance~\cite{artificialanalysis_gptoss120b,clark2025exacto}. 

These considerations motivate the problem of eliciting differences between two LLMs. We approach this issue from the perspective of \emph{distribution testing}: under an appropriate notion of API access to two distributions $\mu, \pi$ over $\Sigma^n$, representing LLMs generating length-$n$ sequences over a token alphabet $\Sigma$, can we estimate the distance between $\mu, \pi$? From a practical perspective, such an estimator could then be used by a user for comparing model providers (how close is the cheap provider to the expensive one?), and conversely, a model provider can use the estimator to test how well a candidate more-efficient inference stack matches their old stack.

In this paper, to evaluate such distributional changes, we focus on measuring the \emph{total variation (TV) distance} between two LLMs. The TV distance, whose definition is recalled in \eqref{eq:distance_def}, is a convenient and interpretable statistical distance. For example, bounding the TV distance also directly implies bounds on the change to any distinguishing test, e.g., the loss on a fixed evaluation metric.

Prior work on evaluating distributional changes in LLMs \cite{fireworks_quantization,DuttaKKR24} often measure the KL divergence, whose definition is also recalled in \eqref{eq:distance_def}. From a practical standpoint, this  has several disadvantages. The KL divergence between distributions is unbounded when their supports are disjoint, so it is not possible to estimate in general to $\eps$ precision from finite samples, as this would require support recovery.
Smoothed approximations to the KL can be used by e.g., introducing pseudocounts, but these strategies do not have a uniform approximation guarantee and require hyperparameter selection. In contrast, the TV distance is bounded in $[0, 1]$, and therefore it inherits an interpretable scale. Moreover, the TV distance is robust to small implementation differences, such as restricting next-token distributions to their top-$k$ elements for $k = 20$ vs.\ $k = 30$, or censoring swear words. Such differences alter the support and thus cause KL to diverge. Measuring KL has thus far remained the practitioner default largely because of the chain rule: the sequence-level KL divergence is the sum of the expected per-token KL divergences, so per-step logit comparisons give unbiased estimators.  TV has no such decomposition, which is the technical obstacle this paper addresses.  
For further discussion on the impact of the choice of distance metric, see Appendix~\ref{sec:TVKL}.

\subsection{Our results}

We describe results under three models of accessing LLMs: sample access, logit access, and noisy distribution access. These are formalized in Definitions~\ref{def:next_token_logprob},~\ref{def:next_token_sample}, and~\ref{def:conditional-sketch}, but we informally describe them here. Throughout, we fix two distributions $\pi$, $\mu$ over $\Sigma^n$ whose TV distance we want to estimate, and we define $m \defeq |\Sigma|$. We also let $K \in [m]$ denote the maximum support size of any next-token distribution $\pi(x_i \mid x_{<i})$ or $\mu(x_i \mid x_{<i})$, where $x_{<i} \in \Sigma^{i - 1}$ is a specified prefix. As one motivation, LLM APIs typically support specifying a $k\in \N^+$ and sampling from the top-$k$ tokens. Due to nondeterminism in which tokens reach the threshold, the total support size $K$ is somewhat larger than $k$, but we find it to be well below $2k$ in practice (Figure~\ref{fig:topk-union} in Appendix~\ref{app:exp-topk-union}).

\paragraph{Sample access.} The simplest and most general access model we consider is \emph{sample access}: we can query any prefix $x_{<i}$ and get a sample from the next-token distribution $\pi(x_i \mid x_{<i})$. This oracle is explicitly studied by \cite{adar2024improved, meel2025distance}, and has an intellectual history in the recent theoretical computer science literature on conditional property testing \cite{chakraborty2013power, acharya2014chasm, canonne2015testing, bhattacharyya2018property}.

\begin{theorem}[Informal, cf.\ Corollary~\ref{cor:blackbox-mlmc}, Corollary~\ref{cor:sampler-lower}]\label{thm:n2upper}
There is an algorithm that estimates $\TV{\pi, \mu}$ up to error $\eps$ with probability $\ge \frac 9 {10}$, using $O(\frac{n^2 K}{\eps^2} \cdot \log^2(\frac 1 \eps))$ queries to prefix sampling oracles to both $\pi$ and $\mu$. Moreover, any algorithm achieving the same guarantee requires $\Omega(\frac{nK}{\eps^2 \log K})$ queries.
\end{theorem}

The most directly-related result to our upper bound in Theorem~\ref{thm:n2upper} is by \cite{meel2025distance}, whose Theorem 1 uses $O(\frac{n^3 m}{\eps^5})$ prefix sampling oracle queries.  Theorem~\ref{thm:n2upper} is better in terms of sequence length $n$, accuracy $\eps$, and it saves an $\frac{m}{K}$ factor when the support is sparse.\footnote{Depending on $m$ rather than $K$ appears inherent to the \cite{meel2025distance} strategy, which involved ``taming'' distributions by padding with a $\frac 1 m$-scaled uniform distribution (see discussion after their Theorem 4.8).}  This $\frac{m}{K}$ factor can be quite large: for reference, Qwen has $m$=151,643 tokens and recommends sampling with $k=20$~\cite{qwen3.5}.

\paragraph{Logit queries.} Beyond providing sample access, many LLM APIs also return the logits: we can query any prefix $x_{<i}$ and see the full vector $\pi(x_i \mid x_{<i})$.  This is much more powerful --- for instance, it lets us compute $\pi(x)$ with $n$ queries --- leading to a better sample complexity.

\begin{theorem}[Informal, cf.\ Theorems~\ref{thm:whitebox_lr},~\ref{thm:prefix-conditional-lower-bound}]\label{thm:n1upper}
There is an algorithm that estimates $\TV{\pi, \mu}$ up to error $\eps$ with probability $\ge \frac 9 {10}$, using $O(\frac n {\eps^2})$ queries to prefix logit oracles to both $\pi$ and $\mu$. Moreover, this oracle query complexity is tight up to constant factors for any algorithm.
\end{theorem}

To our knowledge, we are the first paper to formalize the corresponding TV distance estimation problem under logit access. As discussed, the logit variant of the problem is natural in practical settings where the entire next-token sampling distribution can be queried, and Theorem~\ref{thm:prefix-conditional-lower-bound} provides matching upper and lower bounds for its query complexity.

Our upper bound is proven using a simple observation: $O(n)$ queries to prefix logits allows exact computation of a likelihood-ratio $\frac{\pi(x)}{\mu(x)}$, for any sequence $x \in \Sigma^n$. This allows us to repurpose a TV estimator from \cite{canonne2014testing}, which uses $O(\frac 1 {\eps^2})$ likelihood-ratio queries. We complement this algorithm with a matching distribution over hard instances $\pi$ and $\mu$, indexed by a hidden parameter $p \in [0, 1]$. For this construction, we reduce estimating $\TV{\pi, \mu}$ to learning $p$ up to additive error $\eps$. By carefully defining the hard instance, we show that $\approx n$ prefix logit oracle queries are needed to simulate one toss $\sim \Bern(p)$, at which point $\approx \frac 1 {\eps^2}$ tosses are necessary.

\begin{figure}[t]
    \centering
    \includegraphics[width=\linewidth]{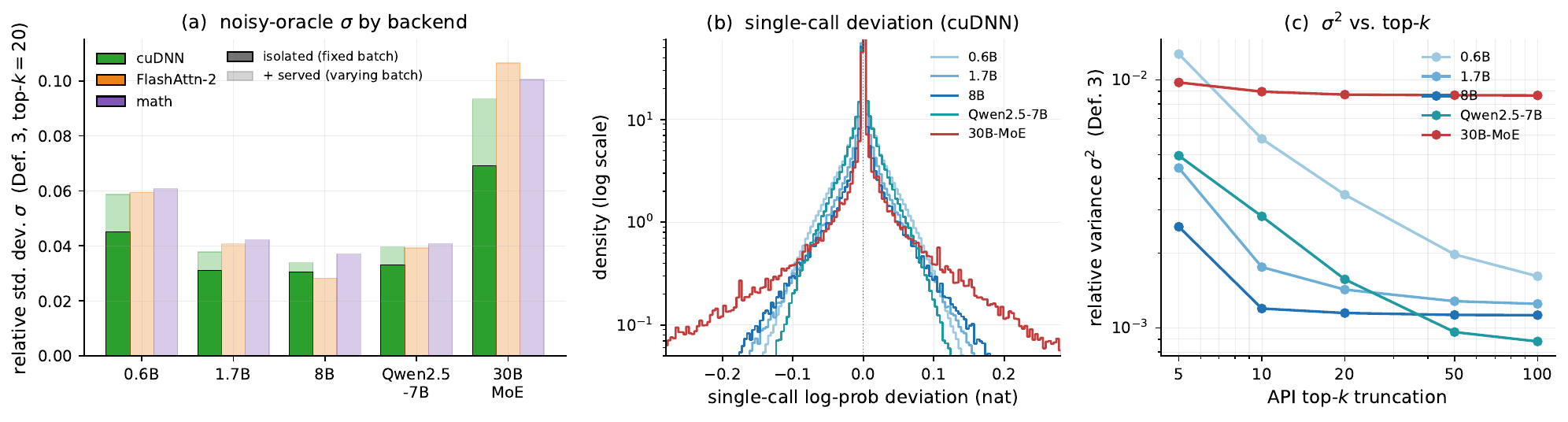}
    \caption{\textbf{Inference engines as noisy logit oracles.} A served model is treated as a per-token log-probability oracle: querying a fixed prefix repeatedly, we measure the relative standard deviation $\sigma$ of the returned probabilities ($\sigma^2$ is the relative-variance parameter of Definition~\ref{def:conditional-sketch}; five Qwen models, bf16, top-$k=20$; calibration in Appendix~\ref{app:exp-sigma-cal}). (a) $\sigma$ by attention backend under two batching conditions: \emph{isolated} (the query alone in a fixed batch; solid) and \emph{served} (the query batched with other, varying requests, as in production; full bar height). cuDNN is noisy even in isolation; FlashAttention-2 and math are deterministic in isolation and acquire noise only when served. (b) The distribution of single-call log-probability deviations. (c) $\sigma^2$ against the API top-$k$ truncation: $\sigma^2 \ll K$ throughout, so the noisy-logit bound $(n + n^2\sigma^2)/\eps^2$ of Theorem~\ref{thm:conditional-sketch-mlmc-intro} is far below the sampling bound $n^2 K/\eps^2$ of Theorem~\ref{thm:n2upper}.}
    \label{fig:oracle-sigma}
\end{figure}

\paragraph{Noisy logits.} Unfortunately, while LLM APIs do return logits, those logits are often not reliable. The issues (batching, kernels, precision) that motivate this paper by making LLM implementations unreliable can also make the computation of logits non-deterministic.  Figure~\ref{fig:oracle-sigma} measures this noise across five models and three attention backends.

To model this, for any given prefix $x_{<i}$ we can no longer query the vector $p := \pi(x_i \mid x_{<i})$ directly.  Instead, we suppose that we can sample from an unbiased estimator with bounded relative variance; that is, we can sample a random vector $\overline{p}$ such that $\E[\overline{p}] = p$ and
\begin{align}
\E_{a \sim p}\Brack{\var{\frac{\overline{p}_a}{p_a}}} &\leq \sigma^2.\label{eq:var}
\end{align}
This is implied, for instance, by estimating the final logits (before the softmax) to additive $\Theta(\sigma^2)$-variance Gaussian noise, see Lemma~\ref{lem:gaussian-logit-relative-variance}.

We believe this to be a fairly reasonable model. The challenge and non-determinism in sampling from an LLM comes from computing the logits, not in taking the final softmax, and whatever process produces the logits and corresponding $\overline{p}$, the average of $\overline{p}$ really will be the true $p$ --- unbiasedness is thus a matter of definition rather than an empirical assumption, and what the experiments document is the noise itself.  The main caveat is that we need independence across samples.

Figure~\ref{fig:denoise} previews the estimator developed in this paper
under this noise model. The two backend pairs in the figure produce
per-token deviations of nearly identical width and cannot be separated by
single queries. Averaging over repeated queries separates them: the
estimated TV between auto and cuDNN decreases with the repetition count
until it reaches the self-noise floor, identifying the two backends as
serving the same distribution (auto dispatches to the cuDNN kernel),
while the estimated TV between math and FlashAttention-2 converges to a
positive value that further repetition does not reduce.

\begin{figure}[t]
    \centering
    \includegraphics[width=0.92\linewidth]{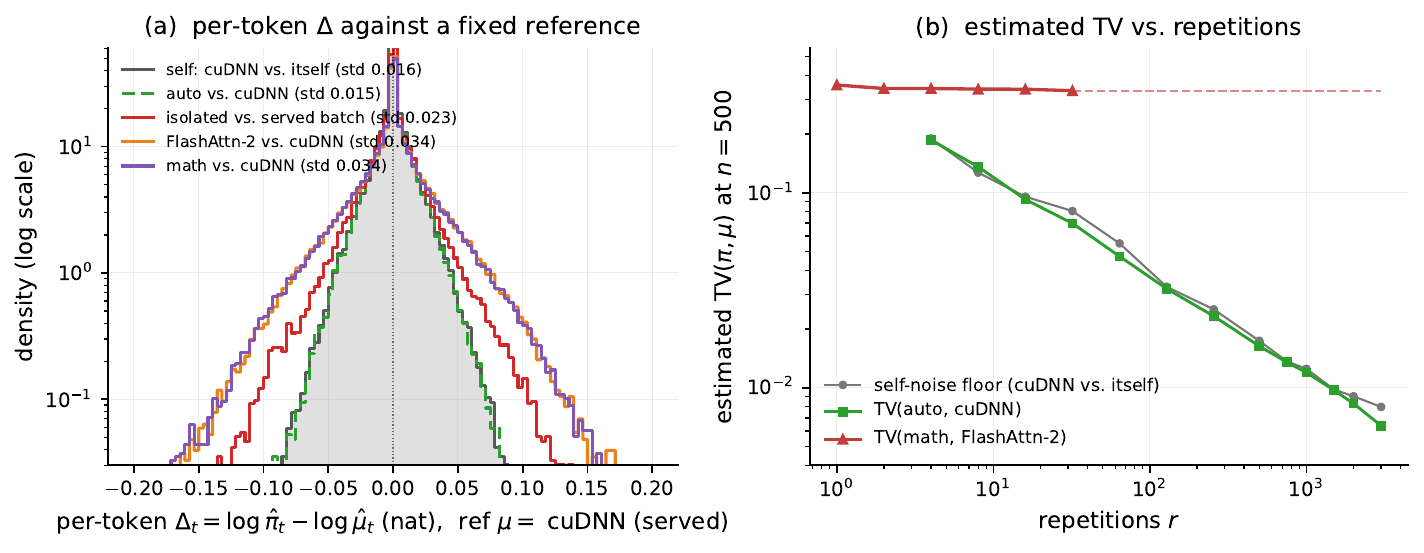}
    \caption{\textbf{Denoising separates true distance from oracle noise.}
(a) Per-token log-probability differences against a fixed reference
(cuDNN, served): the oracle against itself and auto (SDPA's automatic
backend selection) against cuDNN yield the same zero-centered spike, while
pairs that genuinely differ (a changed batching condition;
FlashAttention-2; math) yield broadened distributions. (b) The same
pairs under the TV estimator: as the repetition count $r$ grows,
$\TV{\mathrm{auto}, \mathrm{cuDNN}}$ decays along the \emph{self-noise
floor} (the estimated TV between an oracle and an independent copy of
itself), whereas $\TV{\mathrm{math}, \mathrm{FlashAttention}\text{-}2}$
stabilizes.}
    \label{fig:denoise}
\end{figure}

The ``noisy prefix distribution model'' \eqref{eq:var} (see also Definition~\ref{def:conditional-sketch}) \emph{smoothly interpolates} between the previous models (Lemma~\ref{lem:oracle_reduction}).  If $\sigma=0$, it is precisely the logit query model. With just sample access, we can consider the sample $a \sim p$ as an unbiased one-hot distribution $e_a$, with relative variance
\[
\sigma^2 := \sum_{a\in \text{supp}(p)}  \frac{1}{p_a} \Var[\overline{p}_a] = \sum_{a\in \text{supp}(p)} (1 - p_a) = |\text{supp}(p)|-1 < K.
\]

Our last main result gives an efficient query complexity bound under noisy logit access.
\begin{theorem}[Informal, cf.\ Theorems~\ref{thm:conditional-sketch-mlmc},~\ref{thm:noisy_lb_general}]
\label{thm:conditional-sketch-mlmc-intro}
There is an algorithm that estimates $\TV{\pi, \mu}$ up to error $\eps$ with probability $\ge \frac 9 {10}$, using $O(\frac{n + n^2\sig^2}{\eps^2}\log^2(\frac 1 \eps))$ queries to noisy prefix distribution oracles for both $\pi$ and $\mu$. Moreover, any algorithm achieving the same guarantee requires $\Omega(n\cdot \frac{1\lor \sig^2}{\eps^2 \log K})$ queries.
\end{theorem}

Theorem~\ref{thm:conditional-sketch-mlmc-intro} precisely recovers Theorem~\ref{thm:n2upper}, and it recovers Theorem~\ref{thm:n1upper} up to $\log^2 \frac{1}{\eps}$.

To prove Theorem~\ref{thm:conditional-sketch-mlmc-intro}, we first consider an alternative representation of the TV distance (cf.\ \eqref{eq:tv-mixture-representation} in Lemma~\ref{lem:tv-representations}), which allows us to avoid the ``distribution taming'' step caused by the unbounded estimators used by the \cite{meel2025distance} algorithm. As we discuss in Section~\ref{ssec:est_construct_blackbox}, this simple change already improves the query complexity of \cite{meel2025distance} to $\approx \frac{n^2 K}{\eps^4}$. To achieve our tightest result in Section~\ref{ssec:var_reduce_blackbox}, we use a variance reduction scheme patterned off of multilevel Monte Carlo (MLMC) estimation, a classical technique in statistics \cite{heinrich2001multilevel, giles2015multilevel}. Our approach is particularly inspired by an MLMC scheme from \cite{giles2008multilevel}, which trades off cheaper query complexities with larger variances at different levels of estimation, using a telescoping sum.

Our lower bound is based on a rare-escape embedding of a hard instance for finite-domain TV estimation. Specifically, we embed a hard pair $(P,Q)$ into autoregressive distributions \(\pi_P,\mu_Q\) that usually emit deterministic stay tokens, but with probability \(\Theta(\frac 1 n)\) escape and reveal one draw from \(P\) or \(Q\). This preserves TV distance up to constants, while making each prefix query contain only a \(\Theta(\frac 1 n)\) fraction of the effective information about the underlying finite-domain instance. Since noisy prefix queries to the embedded processes can be simulated from i.i.d.\ samples of \(P,Q\), classical finite-domain lower bound given by \cite{jiao2018minimax} implies that the dependence on support size or noise is unavoidable in our setting. The remaining main gap is therefore in the sequence-length dependence: our lower and upper bounds differ by one additional factor of \(n\).

We summarize our theoretical results in Table~\ref{tab:bounds-summary}, which illustrates the tradeoff between Theorems~\ref{thm:n2upper} and~\ref{thm:n1upper} that is interpolated in Theorem~\ref{thm:conditional-sketch-mlmc-intro}: the stronger logit query model improves the query complexity by a factor of $n$ and eliminates the dependence on $K$. An intriguing question which we leave open for future work is whether this separation is inherent, i.e., if superlinear-in-$n$ lower bounds or subquadratic-in-$n$ upper bounds can be proven for the prefix sampling oracle access model.

\textbf{Experiments.} Beyond the oracle characterization of Figures~\ref{fig:oracle-sigma} and~\ref{fig:denoise}, Section~\ref{sec:case-study} presents a case study of independent interest: the TV distance between two production inference engines, \texttt{vllm} and \texttt{sglang}, serving identical \texttt{Qwen3-0.6B} weights, including how the estimate depends on the serving configuration, and a validation of the multilevel schedule of Algorithm~\ref{alg:blackbox-tv} on real noisy oracles. Appendix~\ref{app:experiments} collects the supporting material: a synthetic validation with known ground truth (Appendix~\ref{app:exp-synthetic}), the calibration of $\sigma$ (Appendix~\ref{app:exp-sigma-cal}), the construction of a faithful scoring oracle and the bias of standard sequence-scoring APIs (Appendix~\ref{app:exp-replay}), and the dependence of the measured distance on the top-$k$ truncation and on the prompt (Appendices~\ref{app:exp-topk-dependence} and~\ref{app:exp-entropy}), as well as an empirical account of the cost of sample-only access (Appendix~\ref{app:exp-access}).

\begin{table}[t]
\centering
\begin{tabular}{lcccc}
\toprule
\textbf{Oracle}&\textbf{Prior upper}&\textbf{Our upper}&\textbf{Prior lower}&\textbf{Our lower}
\\
\midrule
Sampling
&$\displaystyle \frac{n^3 m}{\eps^5}$ \cite{meel2025distance}&$\displaystyle \frac{n^2 K}{\eps^2}$&$\displaystyle \frac{n}{\log n}$ \cite{canonne2020testingbayesiannetworks}&$\displaystyle \frac{nK}{\eps^2}$\\
Logit
&--&$\displaystyle \frac{n}{\eps^2}$&--&$\displaystyle \frac{n}{\eps^2}$
\\[0.6em]

Noisy logit
&--&$\displaystyle \frac{n + n^2 \sigma^2}{\eps^2}$&$\displaystyle \frac{n}{\log n}$ \cite{canonne2020testingbayesiannetworks}&$\displaystyle \frac{n + n\sig^2}{\eps^2}$\\
\bottomrule
\end{tabular}
\vspace{.5em}
\caption{Prefix query complexity bounds for TV estimation over $\Sigma^n$
with logarithmic factors omitted. Here $m \defeq |\Sigma|$, and
$K \le m$ is the maximum support size of any next-token distribution.}
\label{tab:bounds-summary}
\end{table}

\subsection{Related work}
\label{ssec:related_work}

\textbf{TV estimation under conditional access.}
TV estimation is a quantitative analogue of distribution testing: an additive estimator immediately yields tolerant and non-tolerant testers by thresholding. With i.i.d.\ samples, even non-tolerant identity testing over $\{0,1\}^n$ requires $\Omega(2^{n/2})$ samples~\citep{chan2014optimal,valiant2017automatic}, which motivates stronger sampling interfaces. Conditional sampling oracles have been extensively studied in theoretical computer science (TCS), from full conditioning on arbitrary events~\citep{chakraborty2013power} to more restricted oracles such as pair~\citep{acharya2014chasm}, interval~\citep{canonne2015testing,narayanan2020distribution}, subcube~\citep{bhattacharyya2018property, canonne2021random, chakrabarty2025monotonicity}, and prefix~\citep{bhattacharyya2024testing, adar2024improved} oracles. Subcube and prefix models are closest to the autoregressive sampling setting. However, most prior work focuses on non-tolerant testing rather than distance estimation; see~\citep{canonne2020survey} for a broader survey.

The closest theoretical works to ours study TV estimation under prefix sampler access. On the upper-bound side, \cite{meel2025distance,bhattacharyya2024testing, adar2025tight} consider essentially the same sampler oracle as ours. \cite{meel2025distance} show that $O(n^3 m/\eps^5)$ queries suffice, while \cite{bhattacharyya2024testing} improve the bound to $O(n^2 m/\eps^4)$ when one of $\pi,\mu$ is known. Using a different approach, \cite{adar2025tight} obtain an $O(n^2/\eps^4)$ bound, but only for $m=2$. On the lower-bound side, \cite{canonne2020testingbayesiannetworks} prove an $\Omega(n/\log n)$ lower bound for non-tolerant testing, which also implies hardness for TV estimation. In addition, \cite{adar2025tight} prove an $\Omega(n^2/\eps^2)$ lower bound for elementwise probability mass estimation. Although this does not rule out an $o(n^2)$ TV estimator, it suggests that sharper bounds may require new techniques. A separate line of work studies TV estimation with numerical probability access: \citep{canonne2014testing,bhattacharyya2020efficient} assume direct queries to $P(x)$ and $Q(x)$ and show that $\Theta(1/\eps^2)$ queries are necessary and sufficient. Their algorithms can be adapted to our prefix-logit setting, yielding an $O(n/\eps^2)$-query algorithm, but their lower bounds do not apply because their hard instances do not capture  prefix structures.

\textbf{Auditing distribution shifts in LLMs.}
Auditing whether an LLM API serves the model it claims to provide has become a central question in trustworthy model deployment. Some works approach this problem through downstream task performance~\citep{eyuboglu2024model, chen2024chatgpt}. However, task accuracy is only an indirect proxy for model substitution: two models may achieve similar correctness while inducing substantially different output distributions, leading to changes in token usage, response style, or generation quality. A more direct approach is therefore to audit the model-induced output distribution itself, where the core question becomes whether, and by how much, the served distribution differs from the claimed one. Several recent works take this distributional perspective. \cite{gaomodel} proposes a model-equality test, corresponding to non-tolerant testing in TCS, and uses MMD as the test statistic. While empirically promising, this approach lacks theoretical guarantees and was later found to be less effective under subtle substitutions, such as lower-precision inference or partial-time model replacement~\citep{cai2025you}. \cite{zhu2025auditing} strengthens detection via rank-based nonparametric tests, but its complexity and significance analyses rely on assumptions on the log-ranks of model outputs. \cite{arnold2024phase} reported changes in LLM output distributions when varying setups such as prompt formulation and sampling temperature but did not provide explicit finite-sample error bounds for their estimators. Finally, \cite{amini2025better} proposes an algorithm for estimating KL divergence between LLMs, but its analysis assumes bounded estimator variance, which can be very large or even infinite.
\section{Preliminaries}
In this section, we introduce preliminaries used throughout the paper. Section~\ref{ssec:notation} fixes notation; Section~\ref{ssec:oracle_def} defines the oracles used to describe autoregressive models; and Section~\ref{ssec:tv_prop} collects basic properties of TV distance used in our analysis.
\subsection{Notation}
\label{ssec:notation}
\paragraph{Probability.} For a state space $\Omega$, we let $\calP(\Omega)$ denote all probability measures over $\Omega$. For an event $\calE \subseteq \Omega$, we let $\ind(\calE)$ denote the corresponding $0$-$1$ indicator random variable, and $\mu(\calE)$ denote the probability of the event.
We let $\E[\cdot]$ and $\Var[\cdot]$ denote the expectation and variance. We frequently use the following distances (and divergences) between two distributions $\mu, \pi \in \calP(\Omega)$, where $\dd \omega$ denotes a counting measure if $\Omega$ is discrete:
\begin{equation}\label{eq:distance_def}
\begin{aligned}
\TV{\mu, \pi}
&\defeq \half \int_\Omega \Abs{\mu(\omega) - \pi(\omega)} \dd \omega\\
&= \sup_{\calE \subseteq \Omega} \mu(\calE) - \pi(\calE),\\
\KL{\mu}{\pi} &\defeq
    \int_\Omega \log\Par{\frac{\dd \mu}{\dd \pi}\Par{\omega}}
    \mu(\omega)\dd\omega,\\
\CS{\mu}{\pi} &\defeq
    \int_\Omega \Par{\frac{\mu(\omega)}{\pi(\omega)} - 1}^2
    \pi(\omega) \dd \omega.
\end{aligned}
\end{equation}
\paragraph{Autoregressive models.}
Throughout, we let $\Sigma$ be a finite alphabet with $m \defeq |\Sigma|$, and let $\Sigma^n$ denote the space of length-$n$ sequences. In an LLM setting, $\Sigma$ represents a tokenizer vocabulary, and $n$ is a sequence length. When comparing two distributions $\pi$ and $\mu$, we reserve use of the letter $K \in [m]$ to denote the maximum support size of any (prefix-conditional) next-token distribution, i.e., $\pi(X_{i + 1} = \cdot \mid X_1, \ldots, X_{i})$ or $\mu(X_{i + 1} = \cdot \mid X_1, \ldots, X_{i})$. Next-token distributions in $\calP(\Sigma)$ are typically implicitly specified as $\softmax(\ell)$ for logits $\ell \in \R^m$, where $\softmax(\ell) \propto \exp(\ell)$.

Finally we introduce some notation for manipulating sequences. For $x\in\Sigma^n$ and $0\le i\le n$, we write $x_{1:i}=(x_1,\ldots,x_i)$, with $x_{1:0}=\emptyset$. For $\pi\in\mathcal P(\Sigma^n)$, let $\pi_{1:i}$ denote the marginal distribution of $X_{1:i}$. We also use $\pi_{x_{1:i}}(a)$ as shorthand for the next-token distribution $\pi(X_{i + 1} = a \mid X_{1:i} = x_{1:i})$.
For sequences \(s\in\Sigma^m\) and \(s'\in\Sigma^n\), let \(s \oplus s'\in\Sigma^{m+n}\) denote their concatenation. When the domain is clear, we use $e_a$ to denote the $a^{\text{th}}$ basis vector; typically, this is used when $a \in \Sigma$ so that $e_a \in \calP(\Sigma)$.

\subsection{Oracle definitions}
\label{ssec:oracle_def}
We introduce three oracle models for autoregressive distributions, corresponding to different levels of access to an autoregressive model. The strongest model is exact prefix logit access. This corresponds to a white-box setting: one can run the model forward pass at a given prefix and read off the next-token log-probabilities, as if the model weights and inference procedure were directly available.

\begin{definition}[Prefix logit oracle] \label{def:next_token_logprob} Let $\pi\in\mathcal P(\Sigma^n)$. For any prefix $x_{1:i}\in\Sigma^i$ with $0\le i<n$, the prefix logit oracle returns the vector of next-token log-probabilities \footnote{If $\pi(X_{1:i}=x_{1:i})=0$, the conditional distribution is undefined; in this case, one may define the oracle to return a distinguished symbol $\bot$. Throughout the paper, unless explicitly specified, our analysis will not invoke it.} 
\begin{equation} 
\label{eq:exact-logit-query-vec}
\Odist^\pi(x_{1:i})
:= \bigl( \log \pi(X_{i+1}=a\mid X_{1:i}=x_{1:i}) \bigr)_{a\in\Sigma}.
\end{equation} 
We also use the notation 
\begin{equation} 
\label{eq:exact-logit-query-token} 
\Odist^\pi(x_{1:i},a)
:= \log \pi(X_{i+1}=a\mid X_{1:i}=x_{1:i}), \qquad a\in\Sigma.
\end{equation} 
\end{definition}

It is clear that one query to the vector oracle \eqref{eq:exact-logit-query-vec} simulates one scalar oracle query \eqref{eq:exact-logit-query-token} at no additional cost. In the reverse direction, a scalar oracle simulates the full vector oracle using \(m\) scalar queries.  We ignore this alphabet-size distinction in our main complexity statements under prefix logit oracle access.  This convention is natural for LLMs, where a single forward pass produces all logits simultaneously, and log-probabilities differ from logits only by a softmax normalization.

The second model is prefix sampling access. This model is a special case of the more general conditional sampling oracle model that has received substantial attention in the theoretical computer science literature \cite{chakraborty2013power, acharya2014chasm, canonne2015testing, bhattacharyya2018property}. Under this model,
the algorithm cannot directly inspect the model logits, but it can request a token sampled from the model's next-token conditional distribution.  This is also a natural abstraction for closed-source or commercial language models that do not allow explicit access under Definition~\ref{def:next_token_logprob}.

\begin{definition}[Prefix sampler oracle] \label{def:next_token_sample} 
Let $\pi\in\mathcal P(\Sigma^n)$. For any prefix $x_{1:i}\in\Sigma^i$ with $0\le i<n$, the prefix sampler oracle returns 
\[ 
\Osamp^\pi(x_{1:i}) \sim \pi(\cdot\mid X_{1:i}=x_{1:i}). 
\] 
\end{definition}
Finally, we introduce a noisy prefix logit oracle in Definition~\ref{def:conditional-sketch}. This is the most general oracle considered in this paper, subsuming both exact prefix logit queries and prefix sampling queries as special cases. The oracle returns an unbiased random sketch of the next-token distribution, with a noise level parameterized by relative variance. Beyond being a convenient parameterization for analysis, this model captures self-inconsistency in LLM inference, where inference-time randomness can perturb the logits and induce different next-token distributions for the same prefix. 
We substantiate these interpretations in Section~\ref{ssec:relative-variance-motiv}, and highlight the practical relevance of modeling query noise in Figure~\ref{fig:denoise}, where the estimated distance between an autoregressive model and itself is positive until the query noise is averaged away.


\begin{definition}[Noisy prefix distribution oracle] \label{def:conditional-sketch} 
Let $\pi\in\mathcal P(\Sigma^n)$. For any prefix $x_{1:i} \in \Sigma^i$ with $0 \le i < n$, the noisy prefix distribution oracle with relative variance $\sigma^2$ returns an unbiased vector \(\Onoisy^\pi(x_{1:i})\in\calP(\Sigma)\). Writing \(p_i=\pi_{x_{1:i}}\) and \(\widehat p=\Onoisy^\pi(x_{1:i})\), the oracle satisfies $\mathbb E[\widehat p]=p_i$ and
\[
\begin{aligned}
\bbE\Brack{\CS{\widehat p}{p_i}}
&=
\bbE_{a \sim p_i}
\bra{
  \Var\!\left(
    \frac{\widehat p(a)}{p_i(a)}
  \right)
}\\
&\le \sig^2,
\end{aligned}
\]
where $\E$ on the left and $\Var$ on the right are over the internal randomness of $\Onoisy$.
\end{definition}

Lemma~\ref{lem:oracle_reduction} proves that Definition~\ref{def:next_token_sample} is a case of Definition~\ref{def:conditional-sketch} with $\sigma^2 \le K$. Notably, Definition~\ref{def:next_token_sample} is the hardest case of Definition~\ref{def:conditional-sketch}, so we may always assume that $\sigma^2 \in [0, K]$. This is because we can always simulate Definition~\ref{def:next_token_sample} by querying $a \sim \Onoisy(x_{1:i})$ and outputting $e_a$, because $\E[\Onoisy(x_{1:i})] = \pi_{x_{1:i}}$.
We also give broader motivation for our relative variance parameterization in Lemma~\ref{lem:gaussian-logit-relative-variance}.

We use the following unified formulation for our main distance estimation problem under all three oracle models. The
distributions $\pi$ and $\mu$ are only accessed through their corresponding
oracles, and the goal is to estimate their TV distance to additive accuracy
$\varepsilon$.

\begin{problem}[TV estimation]
\label{prob:tv-est}
Let $\pi,\mu\in\mathcal P(\Sigma^n)$.
Assume access to appropriate prefix oracles for $\mu$ and $\pi$, denoted $\calO^\pi, \calO^\mu$. Given $\varepsilon,\delta\in(0,1)$, design an algorithm that outputs $\widehat d$ such that
\[
    \Pr\left[
        \left|\widehat d-\TV{\mu,\pi}\right|\le \varepsilon
    \right]\ge 1-\delta.
\]
The query complexity is the total number of calls to $\calO^\pi$ and $\calO^\mu$.
\end{problem}

Across our analysis, we instantiate $\calO^\pi$ and $\calO^\mu$ with the access models in Definitions~\ref{def:next_token_logprob},~\ref{def:next_token_sample}, and~\ref{def:conditional-sketch}.
\subsection{Properties of TV distance}
\label{ssec:tv_prop}
Here we derive some equivalent formulations of the TV distance used in our estimation algorithms. We use the notation $(x)_+ \defeq \max(x, 0)$.
\begin{lemma}[Equivalent representations of total variation]
\label{lem:tv-representations}
Let $\pi, \mu \in \calP(\Sigma^n)$, then $\TV{\pi, \mu}$ admits the following equivalent representations:
\begin{align}
    \TV{\pi,\mu}
    &=
    \mathbb E_{X\sim\pi}
    \left[
        \left(1-\frac{\mu(X)}{\pi(X)}\right)_+
    \right],
    \label{eq:tv-lr-representation}
    \\[1ex]
    \TV{\pi,\mu}
    &=
    \mathbb E_{X\sim \frac{\pi+\mu}{2}}
    \left[
        \frac{|\pi(X)-\mu(X)|}{\pi(X)+\mu(X)}
    \right].
    \label{eq:tv-mixture-representation}
\end{align}
\end{lemma}

\begin{proof}
By definition,
\begin{equation*}
    \TV{\pi,\mu}
    =
    \frac12\sum_{x\in\Sigma^n}|\mu(x)-\pi(x)|
    =
    \sum_{x\in\Sigma^n}(\pi(x)-\mu(x))_+.
\end{equation*}

We first prove~\eqref{eq:tv-lr-representation} by noting that  
\[
    \mathbb E_{X\sim\pi}
\left[
        \left(1-\frac{\mu(X)}{\pi(X)}\right)_+
    \right]
    =
    \sum_{x\in\supp(\pi)}    \pi(x)
    \left(1-\frac{\mu(x)}{\pi(x)}\right)_+ = \sum_{x \in \Sigma^n}(\pi(x)-\mu(x))_+
    =
    \TV{\pi,\mu}.
\]
A similar direct calculation yields~\eqref{eq:tv-mixture-representation}:
\begin{align*}
\mathbb E_{X\sim \frac{\pi+\mu}{2}}
    \left[
        \frac{|\pi(X)-\mu(X)|}{\pi(X)+\mu(X)}
    \right]
    =
    \frac12\sum_{x\in\Sigma^n}|\pi(x)-\mu(x)|
    =
    \TV{\pi,\mu}.
\end{align*}
\end{proof}
\section{TV Estimation via Prefix Logit Queries}
\label{sec:white_box_tv}
In this section, we study Problem~\ref{prob:tv-est} under access to the prefix logit oracles $(\Odist^\pi,\Odist^\mu)$. We first present our distance estimator and prove its
query-complexity upper bound in Section~\ref{ssec:upper_white_box}. The key observation is that exact prefix logit access simulates an exact log-probability oracle over full trajectories:
for any $x\in\Sigma^n$,
\begin{equation}
    \label{eq:chain_rule_log_prob}
    \ell_\pi(x)
    :=\log \pi(x_{1:n})
    =
    \sum_{i=0}^{n-1}
    \log \pi(x_{i+1}\mid x_{1:i}),
\end{equation}
and a similar expression holds for $\mu$. Hence, density evaluations for $\pi$ or $\mu$ cost $n$
prefix logit oracle queries. We adapt the likelihood-ratio estimator of
\cite{canonne2014testing}, which is a bounded, unbiased estimator for $\TV{\pi, \mu}$ leveraging query access to $\frac{\mu}{\pi}$ via \eqref{eq:tv-lr-representation}, achieving query complexity \(O\left(\frac{n}{\varepsilon^2}\log\frac{1}{\delta}\right).\)

In Section~\ref{ssec:lower_whit_box}, we show that this reduction is tight in the worst case for any algorithm with prefix logit oracle access. In particular, our lower bound construction implies that, when $n\gg 1$, any algorithm solving Problem~\ref{prob:tv-est} with access to $(\Odist^\pi, \Odist^\mu)$ requires \(\Omega(\tfrac{n}{\varepsilon^2})\) queries.

\subsection{Likelihood-ratio estimator}
\label{ssec:upper_white_box}
Our estimator in Algorithm~\ref{alg:tv-test-white-box} follows from the likelihood-ratio representation in equation~\eqref{eq:tv-lr-representation}: it is the empirical mean of a truncated likelihood-ratio functional evaluated on samples from one of the two distributions. To implement Algorithm~\ref{alg:tv-test-white-box}, we need two primitives: (a) sampling a trajectory from one distribution, and (b) evaluating the probability of a given trajectory under either $\pi$ or $\mu$. Both primitives are available from accessing prefix logit oracles. Firstly, sampling can be performed sequentially by querying the prefix logit oracle, and drawing the next token appropriately. Moreover, given a
trajectory $x_{1:n}$, its log-probability is computed via~\eqref{eq:chain_rule_log_prob}.  Each primitive therefore requires $n$ queries. Since the estimator averages bounded independent summands, standard concentration inequalities recalled in Theorem~\ref{thm:whitebox_lr} give the desired high-probability guarantee.
\begin{algorithm}[h]
\caption{$\whiteboxTV(\Odist^\pi,\Odist^\mu,\varepsilon,\delta)$}
    \label{alg:tv-test-white-box}
    \begin{algorithmic}[1]
        \State $N \gets \left\lceil \frac{\log \frac 2 \delta}{2\varepsilon^2}\right\rceil$
        \State Generate $X_1,\ldots,X_N \simiid \pi$ via $\Odist^\pi$
        \For{$i=1,\ldots,N$}
            \State $L_{\pi,i}\gets \log \pi(X_i)$,  $L_{\mu,i} \gets \log \mu(X_i)$ \Comment{Evaluate using $n$ prefix logit oracle queries via \eqref{eq:chain_rule_log_prob}.}
            \State $R_i \gets \left(1-\exp(L_{\mu,i}-L_{\pi,i})\right)_+$
        \EndFor
        \State $\hat d \gets \frac{1}{N}\sum_{i=1}^N R_i$
        \State \Return $\hat d$
    \end{algorithmic}
\end{algorithm}

\begin{theorem}
\label{thm:whitebox_lr}
Algorithm~\ref{alg:tv-test-white-box} solves Problem~\ref{prob:tv-est} with $O(\frac{n}{\varepsilon^2} \log \frac{1}{\delta})$ queries to $(\Odist^\pi, \Odist^\mu)$.
\end{theorem}

\begin{proof}
Letting $\log 0 :=-\infty$, we obtain that  \(R_i=(1-\frac{\mu(X_i)}{\pi(X_i)})_+\). 
By Lemma~\ref{lem:tv-representations}, we know that \(\bbE[R_i]=\TV{\pi, \mu}\). Moreover, since $\cbra{R_i}_{i \in [N]} \in [0, 1]$ are independent, Hoeffding's inequality gives
\[
    \Pr\left[
        \left|
            \frac1N\sum_{i=1}^N R_i-\TV{\pi, \mu}
        \right|
        \ge \varepsilon
    \right]
    \le
    2\exp(-2N\varepsilon^2).
\]
Taking \( N\) as stated gives the desired probability bound. Each sequence log-likelihood evaluation and sample generation uses $n$ queries, which gives us $3nN = O(\frac{n}{\varepsilon^2} \log \frac{1}{\delta})$ query complexity bound.
\end{proof}
\newcommand{\Cond}{\mathsf{Cond}}
\subsection{Lower bound}
\label{ssec:lower_whit_box}
In this subsection, we take $\Sigma=\{0,1\}$ to be a binary alphabet. We provide a hard instance showing that, for constant $\delta$, Theorem~\ref{thm:whitebox_lr} achieves the optimal query complexity up to a constant factor. Intuitively, this hard instance reduces to estimating a parameter $p \in (0, 1)$ up to error $\eps$, in a way such that it takes $\approx n$ prefix logit queries to reveal one sample $\sim \Bern(p)$.

\textbf{Hard instance construction.} Fix $r\in[1,n-2]$, let $t=n-r$, and set $L\defeq |\Sigma^r|=2^r$.  We divide any $x\in\Sigma^n$ into two parts: \(x=u\oplus z\), where \(u\in\Sigma^r,z\in\Sigma^t.\) The prefix $u$ serves as a block index, while the suffix $z$ specifies a
leaf inside that block.  
In each block, we randomly choose whether $\mu_u, \pi_u \in \calP(\Sigma^t)$ are identical or separated. If they are separated, they share the same $(t-1)$-bit prefix inside the block and differ only at the final bit, so the discrepancy takes $t$ queries to reveal.

Formally, for $p\in[0,1]$, define $(\mu,\pi)\sim\mathcal D_p$ as follows. For each $u\in\Sigma^r$, draw
\[
    X_u\simiid\Bern(p),
    \qquad
    Z_u\simiid\Unif(\{0,1\}^{t-1}),
    \qquad
    B_u\simiid\Unif(\{0,1\}).
\]
Let \(a_u\defeq Z_u \oplus B_u\), \(b_u\defeq Z_u\oplus (1-B_u),\) be sibling leaves.  Define distributions $\mu_u,\pi_u \in \calP(\Sigma^t)$ as 
\[
    (\mu_u,\pi_u)
    =
    \begin{cases}
    \left(
        \frac12\delta_{a_u}+\frac12\delta_{b_u},
        \frac12\delta_{a_u}+\frac12\delta_{b_u}
    \right),
    & X_u=0,\\[1em]
    \left(
        \delta_{a_u},
        \delta_{b_u}
    \right),
    & X_u=1.
    \end{cases}
\]
Finally, let \(\mu(u \oplus z)\coloneqq \frac1L \mu_u(z)\), \(\pi(u \oplus z)\coloneqq \frac1L \pi_u(z)\). A sample from $\mu$ is generated by drawing $u\sim\Unif(\Sigma^r)$ and then drawing $z\sim \mu_u$; a sample from $\pi$ is generated analogously using $\pi_u$.

This construction makes the TV distance easy to compute.  Since different prefixes $u$ define disjoint blocks, the global distance decomposes as the average of the suffix distances:
\begin{equation}
    \label{eq:tv-rep-hard-instance}
    \TV{\mu,\pi}
    =
    \frac1L\sum_{u\in\Sigma^r}\TV{\mu_u,\pi_u}=
    \frac1L\sum_{u\in\Sigma^r}X_u.
\end{equation}
The above display shows that, under $(\mu,\pi)\sim\mathcal D_p$, $\TV{\mu,\pi}$ is exactly the empirical average of $L$ i.i.d.\ Bernoulli variables.  Hence, $\TV{\mu,\pi}$ is close to $p$ with high probability.  Therefore, we can show that an accurate estimate of $\TV{\mu,\pi}$ distinguishes the two hypotheses
\[
    H_0:(\mu,\pi)\sim\mathcal D_{p_0},
    \qquad
    H_1:(\mu,\pi)\sim\mathcal D_{p_1},
\]
whenever $p_1-p_0$ is sufficiently large. 
\begin{lemma}
\label{lem:tv-est-to-coin-test}
Let $0\le p_0<p_1\le 1$ and set \(\Delta \coloneqq \frac{p_1-p_0}{4}.\) Suppose $\calA$ estimates $\TV{\mu,\pi}$ within additive error $\Delta$ with probability at least $\frac 2 3$ for every fixed pair $(\mu,\pi)$. Then there is a tester distinguishing $\mathcal D_{p_0}$ from $\mathcal D_{p_1}$ with success probability at least \(\frac23 - 2\exp(-2L\Delta^2)\).
\end{lemma}
\begin{proof}
By~\eqref{eq:tv-rep-hard-instance} and Hoeffding's inequality, we know that for $(\mu, \pi) \sim \mathcal D_p$, 
\[
    \Pr\left[
        |\TV{\mu,\pi}-p|>\Delta
    \right]
    \le
    2\exp(-2L\Delta^2).
\]
The tester runs $\calA$ and outputs $p_1$ iff
\(
    \hat d \ge \tfrac{p_0+p_1}{2}.
\)
Conditioned on the success of the tester and $|\TV{\mu,\pi}-p| \le \Delta$, we have
\(
    |\hat d-p|\le 2\Delta=\frac{p_1-p_0}{2}.
\)
Therefore, if $p=p_0$, then $\widehat d\le \frac{p_0+p_1} 2$; if $p=p_1$, then $\widehat d\ge \frac{p_0+p_1} 2$.  Hence the tester succeeds on these two events.
Applying a union bound gives a success probability of at least \(\frac23 - 2\exp(-2L\Delta^2)\).
\end{proof}

Lemma~\ref{lem:tv-est-to-coin-test} shows it is enough to lower bound the oracle query complexity of estimating $\TV{\mu, \pi}$. Intuitively, we wish to show that many $(u, X_u)$ must be revealed to perform this estimate. 

We then introduce Lemma~\ref{lem:few-revealed-blocks} showing that with only $B$ prefix queries, any adaptive algorithm can reveal at most $O(\frac B t)$ such blocks $u$ in expectation. 
We formalize this through the query-answer transcript. A possibly randomized adaptive algorithm chooses each query $W_i\in\Sigma^{<n}$ as a function of its private randomness and the previous query-answer pairs. We write
\[
    \mathsf{Tr}
    =
    \bigl((W_1,\calO(W_1)),\ldots,(W_T,\calO(W_T))\bigr)
\]
for the full transcript, where $\calO(W_i)$ is the oracle answer and $T \le B$ is the stopping time. For simplicity, we assume that $\calO(W_i)$ reveals both $\Odist^\pi(W_i)$ and $\Odist^\mu(W_i)$, which is without loss of generality up to a constant factor in the oracle query complexity.

The crux of Lemma~\ref{lem:few-revealed-blocks}'s argument is that $X_u$ can affect the transcript only if the algorithm queries the hidden parent $u \oplus Z_u$.  Hence, before learning anything about $X_u$, the algorithm must locate the hidden string
$Z_u\in\{0,1\}^{t-1}$.  We capture this notion formally by bounding the mutual information between the transcript and the initially-random set of strings
$\mathcal Z=(Z_u)_{u\in\Sigma^r}$.

\begin{lemma}
\label{lem:few-revealed-blocks}
Consider any adaptive algorithm interacting with the oracle in Definition~\ref{def:next_token_logprob} on an instance $(\mu, \pi)$ drawn from $\mathcal D_p$. Suppose its worst-case query complexity is $B$. Let $\calG \subseteq \Sigma^r$ be the set of $u \in \Sigma^r$ for which $W_j = u \oplus Z_u$ for some $j \le T$.  There is a universal constant
$C>0$ such that
\[
    \mathbb E[|\mathcal G|]\le C \cdot \frac{B}{t}.
\]
\end{lemma}

\begin{proof}
Throughout, let $q$ denote a random seed specifying internal randomness, which is drawn and fixed before interactions with $\calO$. Thus, the algorithm is deterministic conditioned on $q$.

Let $\ell=t-1 \ge \frac t 2$ and let $\mathcal Z=(Z_u)_{u\in\Sigma^r}$. Inside any block $u \in \Sigma^r$, both $\mu_u$ and $\pi_u$ are supported on the two sibling leaves $Z_u\oplus B_u$ and $Z_u\oplus(1-B_u)$. Hence, before $u \oplus Z_u$ is queried, the oracle only reveals whether the queried prefix is consistent with $Z_u$ and, if so, reveals the next bit of $Z_u$.\footnote{If not, the model returns the distinguished symbol $(\perp, \perp)$, which is included in the transcript and carries only \(O(1)\) bits.} When the hidden $u \oplus Z_u$ is queried, the answer is $(\half(e_0 + e_1),\half(e_0 + e_1))$ if $X_u=0$, and is $(\delta_{B_u},\delta_{1-B_u})$ if $X_u=1$. Therefore $X_u$ can affect the oracle answer only at the query $u \oplus Z_u$.

Let $\mathsf{Tr}$ be the full transcript. We first upper bound the mutual information
$I(\mathcal Z;\mathsf{Tr} \mid s)$.  As shown above, each oracle answer lies in a constant-sized set so each query contributes at most $h_0 = O(1)$ bits of mutual information about $\mathcal Z$. Because the algorithm makes at most $B$ queries, we have \(I(\mathcal Z;\mathsf{Tr}\mid q) \le h_0 B\). Averaging over the internal randomness $q$, we obtain $I(\calZ; \mathsf{Tr}) \le h_0 B$.

Conversely, if $u\in\mathcal G$, then the transcript contains a query of the form $u \oplus Z_u$, so it determines the full string $Z_u$.  Thus, after observing the transcript, revealed blocks $u$ have no remaining entropy in their hidden parents, while each unrevealed block $u$ has at most $\ell$ bits of remaining entropy. Hence
\[
    H(\mathcal Z\mid \mathsf{Tr})
    \le
    \ell\bigl(L-\mathbb E[|\mathcal G|]\bigr).
\]
Since the unconditional entropy $H(\mathcal Z)=L\ell$, we get
\[
    I(\mathcal Z;\mathsf{Tr})
    =
    H(\mathcal Z)-H(\mathcal Z\mid \mathsf{Tr})
    \ge
    \ell\,\mathbb E[|\mathcal G|].
\]
Combining the upper and lower bounds gives
\[
    \mathbb E[|\mathcal G|]
    \le
    \frac{h_0B}{\ell}
    \le
    2h_0\frac{B}{t},
\]
where we used $\ell=t-1\ge \frac t 2$.
\end{proof}
We are now ready to prove the lower bound. Lemma~\ref{lem:tv-est-to-coin-test} shows that a TV estimator for our random instance can distinguish whether the underlying bias is \(p_0\) or \(p_1\).  In the construction, this bias appears only through the variables \(X_u\sim\Bern(p)\); conditioned on these variables, the remaining randomness in \(Z_u\) and \(B_u\) is independent of \(p\).

Thus, given an unknown coin of bias \(p\in\{p_0,p_1\}\), we can simulate the oracle for \((\mu,\pi)\sim\mathcal D_p\) by lazily drawing a coin sample only when the transcript reveals some \(X_u\).  By Lemma~\ref{lem:few-revealed-blocks}, a cost-\(B\) transcript reveals only \(O(\frac B t)\) such blocks in expectation. Hence a cost-\(B\) TV estimator would yield a coin distinguisher using only \(O(\frac B t)\) coin samples. Combining this result with classic sample complexity lower bounds for Bernoulli parameter estimation yields Theorem~\ref{thm:prefix-conditional-lower-bound}.
\begin{theorem}
\label{thm:prefix-conditional-lower-bound}
Fix any $\eps \in (0, \frac 1 8)$, and let $n = \Omega(\log \frac{1}{\eps})$ for a sufficiently large constant. Any algorithm that solves Problem~\ref{prob:tv-est} with probability at least $\frac 2 3$ requires $ \Omega(\tfrac{n}{\eps^2})$ queries to $(\Odist^\pi, \Odist^\mu)$.
\end{theorem}

\begin{proof}
Set \(p_0 = \frac12, p_1=\frac12+4\varepsilon\). By assumption we can set $r = \Omega(\log \frac 1 \eps)$ large enough so that $n \ge 2r$ and the error term in Lemma~\ref{lem:tv-est-to-coin-test} satisfies \(2\exp(-2L\varepsilon^2)\le \frac1{100}\). Suppose that there is an estimator $\calA$ with worst-case prefix-query cost $B$.  By Lemma~\ref{lem:tv-est-to-coin-test}, $\calA$ yields a distinguisher between
\[
    (\mu,\pi)\sim\mathcal D_{p_0}
    \qquad\text{and}\qquad
    (\mu,\pi)\sim\mathcal D_{p_1}
\]
with success probability at least $\frac 2 3 - \frac 1 {100}> \frac 3 5$. 

We now simulate this distinguisher in an equivalent way, using an unknown coin of bias $p\in\{p_0,p_1\}$.  The simulator samples all of the hidden $Z_u$ and $B_u$ using its own randomness, but does not sample $X_u$ in advance.  Whenever the distinguishing algorithm reveals an $X_u$, equivalently when some $W_i = u \oplus Z_u$ in the transcript, the simulator flips the unknown coin and stores it as $X_u$. By lazy sampling, the simulated transcript has exactly the same law as the true transcript under $\mathcal D_p$.

Let $N$ be the number of coin samples used by the simulator.  Since each coin
sample corresponds to one revealed block, Lemma~\ref{lem:few-revealed-blocks}
implies \(\mathbb E[N]\le C\frac{B}{t}\). By Markov's inequality, we may truncate the simulation so that it uses at most $100 C\tfrac{B}{t}$ coin samples while losing only $\frac 1 {100}$ of success probability. Thus, we obtain a tester distinguishing $\Bern(p_0)$ from $\Bern(p_1)$ with success probability $> \frac 3 5$ using $O(\frac B t)$ coin samples. Classical results, e.g., \cite{canonne2022short}, imply that distinguishing $\Bern(p_0)$ from $\Bern(p_1)$ requires \(\Omega\left(\tfrac1{\varepsilon^2}\right)\) samples, which implies $B=\Omega\left(\frac{t}{\varepsilon^2}\right) = \Omega\paren{\frac{n}{\varepsilon^2}}$.
\end{proof}
\begin{remark}[Extension to weaker prefix oracles]
\label{rem:lower-bound-weaker-oracles}
The same lower bound also applies to weaker access models, such as noisy prefix
distribution or prefix sampler oracles. The only oracle-specific part of the proof is Lemma~\ref{lem:few-revealed-blocks}, where we show that a transcript of \(B\) prefix queries can reveal at most \(O(B/t)\) hidden blocks in expectation. This step is information-theoretic: the bottleneck is locating the hidden parent \(u\oplus Z_u\), before which the transcript contains no information about the block variable \(X_u\). Since noisy distribution access and sampler access are no more informative than exact prefix-logit access, the same mutual-information
bound on the revealed set \(\calG\) continues to hold. The remaining reduction from TV estimation to distinguishing \(\mathrm{Bern}(p_0)\) from \(\mathrm{Bern}(p_1)\) is unchanged, and therefore these weaker oracle models inherit the same \(\Omega(n/\varepsilon^2)\) lower bound.
\end{remark}
\section{TV Estimation via Noisy Prefix Distribution Queries}
\label{sec:blackbox-prefix-tv}

In this section, we study Problem~\ref{prob:tv-est} under access to the noisy prefix distribution oracles $(\Osketch^\pi,\Osketch^\mu)$. We begin by motivating the noisy prefix distribution oracle definition in Definition~\ref{def:conditional-sketch} and justifying the claims from Section~\ref{ssec:oracle_def}, particularly, that our results for Definition~\ref{def:conditional-sketch} with $\sigma^2 = k$ capture the prefix sampling oracle in Definition~\ref{def:next_token_sample}. We then present our basic estimator and its approximation guarantee in Section~\ref{ssec:est_construct_blackbox}, followed by an improved variance reduction argument based on multilevel Monte Carlo in Section~\ref{ssec:var_reduce_blackbox} that yields our final query complexity bound. In Section~\ref{ssec:lb_general}, we show that the dependence on the noise level $\sig^2$ and the effective support size $K$ in the upper bound is unavoidable by providing a hard construction, leaving open the dependence on $n$. 

\subsection{Motivation for the relative variance model}
\label{ssec:relative-variance-motiv}
Here we give two basic justifications for the relative variance parameterization used in our noisy prefix distribution oracle (Definition~\ref{def:conditional-sketch}). First, under this parameterization, Definition~\ref{def:conditional-sketch} subsumes the two oracle models introduced earlier: exact logit queries correspond to zero-noise sketches, while prefix-conditional samples correspond to random one-hot probability vectors with $\sigma^2 \le k$. 

\begin{lemma}
    \label{lem:oracle_reduction}
    A prefix logit oracle (Definition~\ref{def:next_token_logprob}) is a noisy prefix distribution oracle (Definition~\ref{def:conditional-sketch}) with $\sigma^2 = 0$. Moreover, drawing $a \in \Sigma$ from a prefix sampler oracle (Definition~\ref{def:next_token_sample}) and outputting $e_a \in \calP(\Sigma)$ is a noisy prefix distribution oracle with $\sigma^2 = k$. 
\end{lemma}
\begin{proof}
The first statement is trivial. For the second, consider any $\pi \in \calP(\Sigma^n)$ and any prefix $s \in \Sigma^{i}$ with $i \in [n-1]$. Then letting $a \sim \Osamp^\pi(s)$, we have $\E e_a = \pi_{s}$ by definition, and
\begin{align*}
\bbE_{a \sim \pi_s}\Brack{\CS{e_a}{\pi_s}} &= \E_{a \sim \pi_s}\Brack{\sum_{b \in \Sigma: [\pi_s]_b > 0} \frac{(\ind_{a = b}- \pi_s(b))^2}{\pi_s(b)}} \\
&= \E_{a \sim \pi_s}\Brack{\frac{(1 - \pi_s(a))^2}{\pi_s(a)}  + 1 - \pi_s(a)} = \E_{a \sim \pi_s}\Brack{\frac{1 - \pi_s(a)}{\pi_s(a)} } \le \Abs{\supp(\pi_s)} \le k.
\end{align*}
\end{proof}

Second, Definition~\ref{def:conditional-sketch} also captures a simple form of noisy logit access with linear noise. Since language models produce logits before applying the softmax, a natural abstraction of inference-time instability is as an additive perturbation to the logits. Lemma~\ref{lem:gaussian-logit-relative-variance} shows that small (additive) Gaussian logit perturbations induce bounded relative variance at the distributional level.

\begin{lemma}
\label{lem:gaussian-logit-relative-variance}
Fix $\ell \in \R^m$, let $s(\xi) := \softmax(\ell+\xi) \in \calP(\Sigma)$, and let $s_a(\xi) = \softmax(\ell+\xi)_a$, for all $a \in \Sigma$. Let $p:=\bbE_{\xi}[s(\xi)]$, where $\xi\sim N(0,\gamma^2 I_m)$ is an isotropic Gaussian. If $\gamma^2<\half$, then
\[
    V :=\bbE_{\xi \sim N(0, \gamma^2 I_m)}\bra{\sum_{a \in \Sigma}
    \frac{(s_a(\xi)-p_a)^2}{p_a}}
    \le
    \frac{2\gamma^2}{1-2\gamma^2}.
\]
\end{lemma}

\begin{proof}
By the Gaussian Poincar\'e inequality (Corollary 2.17, \cite{van2014probability}), for any smooth (infinitely-differentiable) function $f$ on $\R^{m}$, if $\xi \sim N(0, \gamma^2 I_m)$, 
\[
    \Var[f(\xi)]
    \le
    \gamma^2\mathbb E\|\nabla f(\xi)\|_2^2 .
\]
Applying this fact to the scalar function $f(\xi)=s_a(\xi)$ for each $a \in \Sigma$, and using the fact that $\nabla_\xi s_a(\xi)=s_a(e_a-s)$ with $\|e_a-s\|_2^2\le 2$, gives
\(
    \Var[s_a(\xi)]
    \le
    2\gamma^2\bbE\bra{s_a^2(\xi)} .
\)
Thus,
\[
    V = \sum_{a \in \Sigma}\frac{\operatorname{Var}[s_a(\xi)]}{p(a)}
    \le
    2\gamma^2
    \sum_{a \in \Sigma}\frac{\bbE\bra{s_a^2(\xi)}}{p(a)}
    =
    2\gamma^2(V+1).
\]
We conclude the proof by rearranging the above inequality.
\end{proof}

\subsection{Estimator construction}
\label{ssec:est_construct_blackbox}

In this section, we develop a basic estimator for solving Problem~\ref{prob:tv-est} under the noisy access afforded by Definition~\ref{def:conditional-sketch}. A na\"ive application of this estimator yields a suboptimal $\approx \frac{n^2\sig^2}{\eps^4}$ query complexity, which is improved by a variance reduction scheme in Section~\ref{ssec:var_reduce_blackbox}.

To begin, we recall the mixture representation of the TV distance \eqref{eq:tv-mixture-representation} from Lemma~\ref{lem:tv-representations}:
\begin{equation}\label{eq:Zdef}
\begin{aligned}
    \TV{\pi,\mu}
    =
    \mathbb E_{X\sim M}[Z(X)],\qquad  M=\frac{\pi+\mu}{2},
    \\
    Z_{\pi,\mu}(x)
    :=
    \begin{cases}
    \dfrac{|\pi(x)-\mu(x)|}{\pi(x)+\mu(x)},
    & \pi(x)+\mu(x)>0,\\[2ex]
    0,
    & \pi(x)=\mu(x)=0.
    \end{cases}
\end{aligned}
\end{equation}
When $(\pi,\mu)$ is clear from context, we write $Z(x)$ instead of $Z_{\pi,\mu}(x)$. Our strategy is to construct an empirical approximation $\widehat{Z}$ to $Z$ and then bound the expected squared error $\bbE_{X\sim M}[(\widehat Z(X) - Z(X))^2]$. 

As a first step, we explain how to construct an approximation to each of the quantities $\pi(x)$ and $\mu(x)$ required by $Z$. Fix some $r \in \N$ and prefix $s \in \Sigma^{<n}$. We define a random estimator
\begin{equation}\label{eq:psr}
    \widehat p_{s,r}
    :=
    \frac1{r}\sum_{j=1}^{r}Y_{s,j}^\pi \in \calP\Par{\Sigma},\qquad Y_{s,j}^\pi \simiid \Onoisy^\pi\Par{s}.
\end{equation}
In other words, $\hp_{s, r}$ simply queries the noisy prefix distribution oracle $\Onoisy^\pi(s)$ for $r$ times in total, and averages the outputs. For any distribution $\pi \in \calP(\Sigma^n)$ and any $r \in \N$, we can therefore define a corresponding empirical distribution $\hpi_r$ as follows: for a string $x \in \Sigma^n$, we let
\begin{equation}\label{eq:hpir}\hpi_r(x) \defeq \prod_{i = 0}^{n - 1} \hp_{x_{1:i}, r}(x_{i + 1})\end{equation}
be a random estimator for $\pi(x)$, where each of the $\hp_{x_{1:i}, r}(x_{i + 1})$ is estimated via \eqref{eq:psr} using $r$ independent queries to $\Onoisy(x_{1:i})$. Thus, one evaluation of $\hpi_r$ requires $nr$ queries to $\Onoisy$.

We first bound the expected $\chi^2$ divergence between $\pi \in \calP(\Sigma^n)$ and its empirical counterpart $\hpi_r$, where expectations are over the randomness defining $\hpi_r$ (e.g., in the oracle queries).

\begin{lemma}
\label{lem:empirical-tree-chi2}
For any $\pi \in \calP(\Sigma^n)$, let $\widehat \pi_r$ be defined in \eqref{eq:hpir}, where for each prefix $s \in \Sigma^{< n}$, we generate an estimate $\hp_{s, r}$ as in \eqref{eq:psr} independently. Then if $\Onoisy$ is a noisy prefix distribution oracle (Definition~\ref{def:conditional-sketch}) with relative variance $\sig^2$, there is a universal constant $C$ such that for any $r \ge Cn\sigma^2$, 
\[
    \mathbb E\left[
        \chi^2(\widehat \pi_r\|\pi)
    \right]
    \le
    \frac{Cn\sig^2}{r}.
\]
\end{lemma}

\begin{proof}
We observe that almost surely $\widehat\pi_r\ll \pi$, since any token with zero conditional probability under any $\pi(\cdot\mid s)$ receives zero mass under $\hp_{s, r}$. Hence
\[
    1+\chi^2(\widehat\pi_r\|\pi)
    =
    \mathbb E_{X\sim\pi}
    \left[
        \left(
            \frac{\widehat\pi_r(X)}{\pi(X)}
        \right)^2
    \right],
\]
where the expectation also includes the randomness of querying $\Onoisy$.
For any fixed prefix $s \in \Sigma^{< n}$, 
\[
\begin{aligned}
    \mathbb E_{a \sim \pi_s}
    \left[
        \left(\frac{\widehat p_{s,r}(a)}{\pi_s(a)}\right)^2
    \right] =
    \sum_{a:\pi_s(a)>0}
    \frac{\mathbb E[\widehat p_{s,r}(a)^2]}{\pi_s(a)} =
    \sum_{a:\pi_s(a)>0}
    \left(
        \pi_s(a)+\frac{\Var[\widehat p_{s,r}(a)]}{\pi_s(a)}
    \right)  \le
    1+\frac{\sig^2}{r}.
\end{aligned}
\]
Applying this one-step estimate sequentially along the path
$X_1,\ldots,X_n$, and using independence of the estimators $\hp_{s, r}$ across
prefixes $s$, gives the claim:
\[
    \E_{X \sim \pi}\left[\left(\frac{\widehat\pi_r(X)}{\pi(X)}\right)^2 \right]
    \le
    \left(1+\frac{\sig^2}{r}\right)^n \le 1 + \frac{Cn\sig^2}{r}.
\]
\end{proof}
In implementing the estimator $\hpi_r$ in \eqref{eq:hpir}, it is helpful to assume that in the background, a data structure has queried $\Onoisy(s)$ for $r$ times for \emph{every prefix} $s \in \Sigma^{< n}$. When our algorithm requires an estimate $\hpi_r(x)$, the data structure lazily provides the estimates $\hp_{s, r}$ for each $s = x_{1:i}$, and we use these estimates to compute $\hpi_r(x)$. We similarly define an estimator $\hmu_r$ for each $r \in \N$.

We are now ready to define our empirical approximation $\hZ$ to $Z$ in \eqref{eq:Zdef}. For fixed $r \in \N$, define $\hpi_r$ and $\hmu_r$ as in \eqref{eq:hpir}, which we have query access to. We then let
\begin{equation}\label{eq:hzdef}
\hZ(x) \defeq \frac{|\hpi_r(x) - \hmu_r(x)|}{\hpi_r(x) + \hmu_r(x)}
\end{equation} 
for any $x$ where $\pi(x) + \mu(x) > 0$. If we observe $\hpi_r(x) + \hmu_r(x) = 0$, we set the corresponding $\hZ(x) = 0$. In Lemma~\ref{lem:contrast-stability} and~\ref{lem:blackbox-l2-contrast}, we reduce bounding the squared error between $\hZ$ and $Z$ to understanding the $\chi^2$ divergences between $(\hpi_r, \pi)$ and $(\hmu_r, \mu)$.
\begin{lemma}
\label{lem:contrast-stability}
For $u,v\ge 0$, let $\phi(u, v) := \tfrac{|u - v|}{u+v}$.
Then for all $a,b,c,d\ge 0$ with $a+b>0$,
\[
    |\phi(c,d)-\phi(a,b)|
    \le
    4\Par{\frac{|c-a|+|d-b|}{a+b}}.
\]
\end{lemma}

\begin{proof}
Let $s:=a+b$, $t:=c+d$, and $\Delta:=|c-a|+|d-b|$.  If $t\ge \frac s 2$, then
\[
\begin{aligned}
    |\phi(c,d)-\phi(a,b)|
    &\le
    \left|
        \frac{c-d}{t}-\frac{a-b}{s}
    \right| \\
    &\le
    \frac{|(c-d)-(a-b)|}{t}
    +
    |a-b|\frac{|t-s|}{st}        \\
    &\le
    \frac{\Delta}{t}+\frac{\Delta}{t}
    \le
    \frac{4\Delta}{s}.
\end{aligned}
\]
We conclude the proof by noting that when $t< \frac s 2$, $\Delta>\frac s 2$, while
$|\phi(c,d)-\phi(a,b)|\le 1$. 
\end{proof}

\begin{lemma}
\label{lem:blackbox-l2-contrast}
Let $\Osketch^\pi$, $\Osketch^\mu$ respectively be noisy prefix distribution oracles with relative variance $\sigma^2$ for $\pi, \mu \in \calP(\Sigma^n)$. There is a universal constant $C>0$ such that, for every $r \ge Cn\sigma^2$, defining $M$, $Z$ as in \eqref{eq:Zdef} and $\hZ$ as in \eqref{eq:hzdef},
\[
    \mathbb E_{X \sim M}\left[
        \left(\hZ(X)-Z(X)\right)^2
    \right]
    \le
    \frac{Cn\sigma^2}{r}.
\]
\end{lemma}

\begin{proof}
Observe that for $X \sim M$, we must have $\pi(X)+\mu(X)>0$ by definition.  Applying Lemma~\ref{lem:contrast-stability} with
\(
    (a,b,c,d)
    \gets 
    (\pi(X),\mu(X),\widehat\pi_{r}(X),\widehat\mu_{r}(X))
\)
shows that
\[
    |\hZ(X)-Z(X)|
    \le
    4\Par{\frac{
        |\widehat\pi_{r}(X)-\pi(X)|
        +
        |\widehat\mu_{r}(X)-\mu(X)|
    }{\pi(X) + \mu(X)}}.
\]
Therefore, using $(a + b)^2 \le 2a^2 + 2b^2$ and $M = \frac{\pi + \mu}{2}$,
\[
\begin{aligned}
    \mathbb E_{X\sim M}
    \left[
        (\hZ(X)-Z(X))^2
    \right]
    &\le
    16\sum_{x\in\Sigma^n}
    \frac{
        (\widehat\pi_{r}(x)-\pi(x))^2
        +
        (\widehat\mu_{r}(x)-\mu(x))^2
    }{\pi(x)+\mu(x)} \\
    &\le
    16\chi^2(\widehat\pi_{r}\|\pi)
    +
    16\chi^2(\widehat\mu_{r}\|\mu).
\end{aligned}
\]
Taking expectations over the randomness of $\hpi_r$, $\hmu_r$ and applying Lemma~\ref{lem:empirical-tree-chi2} proves the claim.
\end{proof}

Lemma~\ref{lem:blackbox-l2-contrast} already implies a nontrivial algorithm for solving Problem~\ref{prob:tv-est}. In particular, taking $r = O(\frac{n\sigma^2}{\eps^2})$ in Lemma~\ref{lem:blackbox-l2-contrast} implies that $\hZ$ has bias $\eps$ with respect to $Z$, and each evaluation of $\hZ$ requires $O(nr)$ queries. Combining this calculation with the $O(\frac 1 {\eps^2})$ evaluations of $\hZ$ required to estimate its expectation to $\eps$ additive error gives a total query complexity of
\[O\Par{\frac{n^2\sigma^2}{\eps^2}} \cdot O\Par{\frac 1 {\eps^2}} = O\Par{\frac{n^2\sigma^2}{\eps^4}}.\]
This bound already improves upon the $O(\frac{n^3}{\eps^5})$ queries required by the prior work of \cite{meel2025distance} in the regime $m = |\Sigma| = O(1)$. In Section~\ref{ssec:var_reduce_blackbox}, we show how to use the squared error bound in Lemma~\ref{lem:blackbox-l2-contrast} in a black-box way within a variance reduction scheme to improve this complexity bound.

\begin{remark}[Asymmetric noise]\label{rem:noise}
Lemma~\ref{lem:blackbox-l2-contrast} generalizes straightforwardly to settings where there is asymmetry in the relative variance bounds for $\Onoisy^\pi$, $\Onoisy^\mu$. In particular, if these bounds are respectively $\sigma_\pi^2$, $\sigma_\mu^2$, and we instead use $r_\pi \in \N$ oracle queries at each prefix to define $\hpi_{r_\pi}$ (similarly, $r_\mu \in \N$ queries to define $\hmu_{r_\mu}$), it is simple to extend Lemma~\ref{lem:blackbox-l2-contrast} to yield an expected squared error of
\[\frac{Cn\sigma_\mu^2}{r_\mu} + \frac{Cn\sigma_\pi^2}{r_\pi}.\]
This asymmetry could be useful in saving samples when, e.g., we have exact query access to one of the autoregressive models $(\pi, \mu)$, but we have noisy or sample access to the other. To ease notational burden, we assume $\sigma_\mu = \sig_\pi$ for simplicity henceforth.
\end{remark}

\subsection{Variance reduction through multilevel Monte Carlo}
\label{ssec:var_reduce_blackbox}
In this subsection, we improve upon a na\"ive application of Lemma~\ref{lem:blackbox-l2-contrast} by estimating a telescoping sequence of increasingly-accurate approximations in a \emph{multilevel Monte Carlo} (MLMC) variance reduction scheme. The analysis of our variance-reduced scheme uses Lemma~\ref{lem:blackbox-l2-contrast} in a black-box way.

Fix an integer $L = O(\log \frac 1 \eps)$ in the following discussion, and for all levels $\ell = 0, 1, \ldots, L$, define
\begin{equation}
    \label{eq:def_r_l}
    r_\ell
    :=
    1 + \left\lceil cn\sig^2\,2^\ell\right\rceil,
\end{equation}
where $c$ is a large constant and $\sigma$ parameterizes the relative variance of $\Onoisy^\pi$, $\Onoisy^\mu$. 
The key to our variance reduction is to couple the external draw $X \sim M$ in adjacent levels, so that the correction from level $\ell-1$ to level $\ell$ has variance proportional to its approximation error. 

More concretely, for all $\ell = 0, 1, \ldots, L$, let $Z_\ell$ denote a random independent copy of the estimator $\hZ$ in \eqref{eq:hzdef}, with $r \gets r_\ell$ as set in \eqref{eq:def_r_l}. Also, define
\begin{equation*}
    \theta_\ell:=\mathbb E[Z_\ell(X)],
    \qquad
    \theta:=\TV{\pi,\mu} = \E_{X \sim M}\Brack{Z(X)},
\end{equation*}
where whenever we take expectations over $Z_\ell$, we also include the randomness from the noisy queries to $\Onoisy$.
By Lemma~\ref{lem:blackbox-l2-contrast}, choosing a large enough $c$ yields
\[
    \mathbb E\left[(Z_\ell(X)-Z(X))^2\right]
    \le
    C2^{-\ell}.\]
Therefore, by the Cauchy-Schwarz inequality,
\begin{equation}
\label{eq:level-bias-bound}
    |\theta_\ell-\theta|
    =
    \left|\mathbb E[Z_\ell(X)-Z(X)]\right|
    \le
    \sqrt{\mathbb E[(Z_\ell(X)-Z(X))^2]}
    \le
    C2^{-\frac \ell 2},
\end{equation}
so that if $L = O(\log \frac 1 \eps)$ is large enough, $|\theta_L - \theta| \le \frac \eps 2$. Our goal is thus to estimate $\theta_L$ to $\frac \eps 2$ additive error. To do so, we write it as a telescoping sum:
\begin{equation}\label{eq:telescope}
    \theta_L =\theta_0+\sum_{\ell=1}^L(\theta_\ell-\theta_{\ell-1}).
\end{equation}
Our goal is now to design estimators for each element $\theta_\ell - \theta_{\ell - 1}$ within the sum \eqref{eq:telescope}. To do so, we couple the outer draw $X \sim M$ used to define these estimators. In particular, define $Z_{-1}(\cdot) \equiv 0$ for simplicity, and for each $\ell = 0, 1, \ldots, L$, define an estimator
\begin{equation}\label{eq:yl_def}Y_\ell \defeq Z_\ell(X) - Z_{\ell - 1}(X),\text{ where } X \sim M. \end{equation}
It is clear that this construction yields
\[\E[Y_0] = \theta_0,\qquad \E[Y_\ell] = \theta_\ell - \theta_{\ell - 1} \text{ for } \ell = 0, 1, \ldots, L.\]
We now bound the variance of each estimator $Y_\ell$.
\begin{lemma}
\label{lem:blackbox-level}
For every $\ell = 0, \ldots, L$, we have $\Var[Y_\ell] \le \min(4, C2^{-\ell})$ for a universal constant $C > 0$. Moreover, we can obtain an i.i.d.\ copy of $Y_\ell$ using $O(n + n^2 \sig^2 2^\ell)$ queries to $\Onoisy^\pi$ and $\Onoisy^\mu$.

\end{lemma}
\begin{proof}
The bound $\Var[Y_\ell] \le 4$ is immediate because each $Z_\ell(X)$ lies in $[0, 1]$ by definition, so each $Y_\ell$ lies in $[-1, 1]$. For the remaining bound, letting $\ell \ge 1$,
\begin{equation}\label{eq:cancel_z}
    Y_\ell
    =
    Z_\ell(X)-Z_{\ell-1}(X)
    =
    (Z_\ell(X)-Z(X))+(Z(X)-Z_{\ell-1}(X)),
\end{equation}
where $X \sim M$. Thus, applying $(a + b)^2 \le 2a^2 + 2b^2$,
\[
\begin{aligned}
    \Var[Y_\ell]
    &\le
    \mathbb E[Y_\ell^2] \le 
    2\mathbb E[(Z_\ell(X)-Z(X))^2]
    +
    2\mathbb E[(Z_{\ell-1}(X)-Z(X))^2] \le
    C2^{-\ell},
\end{aligned}
\]
where the last step follows from adjusting the constant $C$ in Lemma~\ref{lem:blackbox-l2-contrast}. The query complexity is immediate from the definition of $\hZ$ and our sample bound $r_\ell$ in \eqref{eq:def_r_l}.
\end{proof}
The proof of Lemma~\ref{lem:blackbox-level} reveals why it is important that the same $X \sim M$ is used to define $Z_\ell$ and $Z_{\ell - 1}$ in the coupled increment $Y_\ell$. Indeed, if different samples were used, then the identity \eqref{eq:cancel_z} would not hold and we would not inherit the separate squared error bounds from Lemma~\ref{lem:blackbox-l2-contrast}.

Lemma~\ref{lem:blackbox-level} shows that the coupled increments $Y_\ell$ have
geometrically-decaying variance across levels. This suggests a non-uniform
sampling strategy, using more i.i.d.\ draws of $Y_\ell$ on coarse levels (small $\ell$),
where the variance is larger, and fewer samples on fine levels (large $\ell$), where each sample is more expensive. By balancing the
per-level variance contribution against the per-level sampling cost, we obtain
an MLMC estimator that reduces the total
query complexity. 

We now explain how to allocate samples across levels.  Let
\begin{equation}
\label{eq:def_N_l}
    N_\ell
    :=
    \left\lceil
        A \cdot \frac{L+1}{\eps^2} \cdot 2^{-\ell}
    \right\rceil,
    \qquad
    \ell=0,1,\ldots,L,
\end{equation}
where $A>0$ is a universal constant.  Let
$Y_{\ell,1},\ldots,Y_{\ell,N_\ell}$ be independent copies of $Y_\ell$, and
define 
\begin{equation}
\label{eq:def-mlmc-estimator}
     T
    :=
    \frac1{N_0}\sum_{j=1}^{N_0}Y_{0,j}
    +
    \sum_{\ell=1}^L
    \frac1{N_\ell}\sum_{j=1}^{N_\ell}Y_{\ell,j}.
\end{equation}
Our final estimator simply takes the median of $\log(\frac 1 \delta)$ copies of $T$ defined in \eqref{eq:def-mlmc-estimator}. We provide pseudocode for the overall method in Algorithms~\ref{alg:coupled-increment} and~\ref{alg:blackbox-tv}, and analyze its correctness in Theorem~\ref{thm:conditional-sketch-mlmc}.

\begin{algorithm}[h]
\caption{$\csketch(\Osketch^\pi,\Osketch^\mu,\ell, \sigma,n)$}
\label{alg:coupled-increment}
\begin{algorithmic}[1]
    \State $X \sim M = \half(\pi + \mu)$ \Comment{Uses $n$ queries to $\Osketch^\nu$ for $\nu \sim_{\text{unif.}} \{\mu, \pi\}$, as $\Osketch^\nu$ simulates $\Osamp^\nu$}
    \State $Z_\ell(X) \gets $ copy of $\hZ$ in \eqref{eq:hzdef} with $r \gets r_\ell$ in \eqref{eq:def_r_l} \Comment{Uses $nr_\ell$ queries to $(\Osketch^\pi, \Osketch^\mu)$}
    \If{$\ell = 0$}
    \State $Z_{\ell - 1}(X) \gets 0$
    \Else
    \State $Z_{\ell - 1}(X) \gets $ copy of $\hZ$ in \eqref{eq:hzdef} with $r \gets r_{\ell - 1}$ in \eqref{eq:def_r_l} \Comment{Uses $nr_{\ell - 1}$ queries to $(\Osketch^\pi, \Osketch^\mu)$}
    \EndIf
    \State \Return $Z_\ell(X)-Z_{\ell-1}(X)$
\end{algorithmic}
\end{algorithm}

\begin{algorithm}[h]
\caption{$\tvnoise(\Osketch^\pi,\Osketch^\mu,\eps,\delta,\sig,n)$}
\label{alg:blackbox-tv}
\begin{algorithmic}[1]
    \State $R\gets \lceil 10\log \frac 1 \delta\rceil$
    \State $L\gets \left\lceil C+2\log_2 \frac 1 \eps \right\rceil$
    \For{$s=1,\ldots,R$}
        \State $T_s\gets 0$ \Comment{Running sum for $s^{\text{th}}$ evaluation of $T$ in \eqref{eq:def-mlmc-estimator}}
        \For{$\ell=0,\ldots,L$}
            \State $S_\ell\gets 0$
            \For{$j=1,\ldots,N_\ell$}
                \State $Y_{\ell,j}\gets\csketch(\Onoisy^\pi,\Onoisy^\mu,\ell, \sigma,n)$
                \State $S_\ell\gets S_\ell+Y_{\ell,j}$
            \EndFor
            \State $T_s\gets T_s+\frac{S_\ell}{N_\ell}$ where $N_\ell$ is as in \eqref{eq:def_N_l}
        \EndFor
    \EndFor
    \State \Return $\mathrm{median}\{T_1,\ldots,T_R\}$
\end{algorithmic}
\end{algorithm}

We are now ready to prove Theorem~\ref{thm:conditional-sketch-mlmc-intro}, stated slightly more formally here:

\begin{theorem}
\label{thm:conditional-sketch-mlmc}
If $(\Osketch^\pi, \Osketch^\mu)$ are respectively noisy prefix distribution oracles with relative variance $\sigma^2$ for $(\pi, \mu)$,
Algorithm~\ref{alg:blackbox-tv} solves Problem~\ref{prob:tv-est} with $O(\frac{n + n^2\sig^2}{\eps^2}\log^2(\frac 1 \eps)\log(\frac 1 \delta))$
queries to $(\Osketch^\pi, \Osketch^\mu)$.
\end{theorem}

\begin{proof}
We follow the notation in Algorithm~\ref{alg:blackbox-tv}. Observe that for any $s = 1, \ldots, R$, we have $\E[T_s] = \theta_L$ by using the calculation \eqref{eq:telescope}. Thus, for a large enough constant in the choice of $L$, \eqref{eq:level-bias-bound} implies $|\E[T_s] - \theta| \le \frac \eps 2$. Moreover, by independence across calls to $\csketch$, we have by applying the variance bound in Lemma~\ref{lem:blackbox-level} that
\[
\begin{aligned}
    \Var[T_s]
    &=
    \sum_{\ell=0}^L
    \frac{\Var[Y_\ell]}{N_\ell} \le
    \frac4 {N_0}
    +
    \sum_{\ell=1}^L
    \frac{C2^{-\ell}}{N_\ell}.
\end{aligned}
\]
Using $N_\ell$ as defined in~\eqref{eq:def_N_l}, each summand is at most $O(\frac{\eps^2}{A(L+1)})$. Hence, by choosing $A$ sufficiently large, we obtain that $\Var[T_s]\le \frac{\eps^2}{12}$. Chebyshev's inequality then gives the single-estimate bound
\[
\begin{aligned}
    \Pr\left[
        |T_s-\TV{\pi,\mu}|>\eps
    \right]
    &\le
    \Pr\left[
        |T_s-\theta_L|>\frac{\eps}{2}
    \right] \le
    \frac{4\Var[T_s]}{\eps^2}
    \le
    \frac13.
\end{aligned}
\]
Finally, a standard application of Hoeffding's inequality shows that the median of $R$ independent copies of an estimator $T$ obeying the above bound lies within $\eps$ of $\theta = \TV{\pi, \mu}$ with probability at least $1 - \delta$. The query complexity follows from Lemma~\ref{lem:blackbox-level} with our definition \eqref{eq:def_N_l}:
\begin{align*}
    \sum_{\ell=0}^L
    RN_\ell\cdot O(n + n^2\sig^2 2^\ell)
    =
    O\left(
        \frac{(n + n^2\sig^2)(L+1)^2\log \frac{1}{\delta}}{\eps^2}
    \right)
    +
    O\paren{(n+n^2\sig^2)2^L\log\frac
    1\delta},
\end{align*}
where we conclude by using that $2^L = O(\frac 1 {\eps^2})$ by our choice of $L$.
\end{proof}

As a corollary of Theorem~\ref{thm:conditional-sketch-mlmc} and the oracle reduction in Lemma~\ref{lem:oracle_reduction}, we obtain a query complexity bound for Problem~\ref{prob:tv-est} under prefix sampling oracle access.

\begin{corollary}
\label{cor:blackbox-mlmc}
Algorithm~\ref{alg:blackbox-tv} solves Problem~\ref{prob:tv-est} with 
\(
    O(
        \frac{n^2K}{\eps^2}
        \log^2(\frac1\eps)
        \log\frac1\delta
    )
\) queries to $(\Osamp^\pi, \Osamp^\mu)$.
\end{corollary}

\subsection{Lower bound}
\label{ssec:lb_general}
In Section~\ref{ssec:lower_whit_box}, we proved an information-theoretic $\Omega(\frac{n}{\eps^2})$ lower bound for Problem~\ref{prob:tv-est} under prefix logit access. Since exact logit access is more informative than noisy conditional access, this immediately implies the same lower bound for noisy oracles. However, this argument does not explain whether the dependence on the noise level $\sig^2$ and the effective support size $K$ (i.e., in the case of prefix sampler oracles) in the upper bound of Section~\ref{ssec:var_reduce_blackbox} is unavoidable.

In this section, we reduce our noisy-prefix estimation problem to classical finite-domain TV estimation under i.i.d.\ sample access. The reduction is two-fold. First, we embed any marginal distributions \(P,Q \in \calP(\Omega)\) into rare-escape autoregressive distributions \(\pi_P,\mu_Q\): the process usually emits a deterministic ``stay'' token, and with probability \(\Theta(\frac 1 n)\) escapes and reveals one draw from \(P\) or \(Q\). This preserves the target up to constants, \(\TV{\pi_P,\mu_Q}=\Theta(\TV{P,Q})\). Second, we show that calibrated noisy prefix queries to \(\pi_P,\mu_Q\) can be simulated using i.i.d. samples from \(P,Q\). Since useful information appears only through the rare escape event, each prefix query carries only \(\Theta(\frac 1 n)\) effective samples. Hence a fast noisy-prefix estimator would imply a fast i.i.d.\ estimator for \(\TV{P,Q}\), contradicting a known single-marginal lower bound on distance estimation in finite domains, unless the number of prefix queries is larger by a factor of \(n\).

We first recall the classical sample complexity lower bound for TV estimation problem between two distributions on a finite domain in~\cite{jiao2018minimax}.
\begin{proposition}[Theorem 1, \cite{jiao2018minimax}]
\label{prop:dist_est_supp_k}
For every \(m\ge 2\) and \(\eta\in(0,\half)\), any algorithm that solves Problem~\ref{prob:tv-est} for arbitrary \(P,Q\in\calP(\Omega)\) with $|\Omega| = m$ under i.i.d.\ sample access with constant success probability must use \(\Omega(\frac{m}{\eps^2\log m})\) total samples. 
\end{proposition}

\paragraph{Hard case construction.} We now embed the finite-domain problem into our autoregressive TV estimation task. Suppose that $m \ge K \ge 3, n \ge 2$. Let $\Sigma = \cbra{0} \cup [m]$ and $\Omega=[K-1]$. Fix arbitrary $P,Q\in\calP(\Omega)$. We only put probability mass on the sufficient support $\tilde{\Omega} = \cbra{0} \cup \Omega$ and treat $0$ as a stay token and symbols in $\Omega$ as escape labels. Set
\[
    q:=\frac1n,
    \qquad
    \alpha_n := 1-(1-q)^n .
\]
Note that $\alpha_n=\Theta(1)$ for $n\ge 2$. We construct two autoregressive distributions $\pi_P,\mu_Q\in\calP(\Sigma^n)$.
For every alive prefix $0^i$, $0\le i\le n-1$, define
\[
    \pi_P(0\mid 0^i)=\mu_Q(0\mid 0^i)=1-q, 
    \quad     
    \pi_P(a\mid 0^i)=qP(a),
    \quad
    \mu_Q(a\mid 0^i)=qQ(a), \quad \forall a \in \Omega.
\]
In other words, at each all-zeroes prefix, with probability $1 - q$ the next sample is a $0$, and with probability $q$ it draws from the relevant finite-domain distribution $P$ or $Q$.
After the process emits an escape label $a\in\Omega$, all remaining tokens are deterministically $0$ under both $\pi_P$ and $\mu_Q$. On prefixes with zero probability under both laws, we define the two conditional distributions identically; such answers contain no information about $P$ or $Q$ and can only make the instance easier.

The only strings with nonzero probability are $0^n$ and strings of the form
\[
    0^i \oplus a \oplus 0^{n-i-1},
    \qquad
    i=0,\ldots,n-1,\quad a\in\Omega .
\]
Both $\pi_P$ and $\mu_Q$ assign probability $(1-q)^n$ to $0^n$. For $s=0^i\oplus a\oplus 0^{n-i-1}$, we have
\[
    \pi_P(s)=(1-q)^i q P(a),
    \qquad
    \mu_Q(s)=(1-q)^i q Q(a).
\]
Therefore,
\begin{align*}
    \TV{\pi_P,\mu_Q}
    &=
    \frac12\sum_{s\in\Sigma^n}\Abs{\pi_P(s)-\mu_Q(s)}
    =
    \frac12\sum_{i\in[n-1]}(1-q)^i q
    \sum_{a\in\Omega}\Abs{P(a)-Q(a)} \\
    &=
    \paren{1-(1-q)^n}\TV{P,Q}
    =
    \alpha_n\TV{P,Q}.
\end{align*}
Since  $\alpha_n = \Theta(1)$, any estimator for $\TV{\pi_P, \pi_Q}$ can be transferred into an estimator of $\TV{P, Q}$ with only a factor of constant loss. 

\paragraph{Simulating noisy prefix queries from i.i.d.\ samples.}
We then give an explicit way to simulate the noisy prefix oracles of $\pi_P, \mu_Q$ with sample oracles of $P, Q$. Set \(M:=\left\lceil \frac{K-1}{\sig^2}\right\rceil \).  At an alive prefix, letting $e_\omega$ be the point mass on state $\omega \in \Sigma$, the true next-token distributions are
\[
    r_P := (1-q)e_0 + qP,
    \qquad
    r_Q := (1-q)e_0 + qQ .
\]
A noisy query to \(\pi_P\) returns the empirical distribution
\[
    Y^\pi := \frac1M\sum_{j\in[M]} e_{A_j},
    \qquad
    A_j\simiid r_P.
\]
Equivalently, each \(A_j\) can be generated by first flipping a \(\mathrm{Bern}(q)\) coin: if it is \(0\), output the stay-alive symbol \(0\); otherwise draw one sample from \(P\). Thus one alive query to \(\pi_P\) uses at most \(M\) i.i.d.\ samples from \(P\). Similarly, a noisy query to \(\mu_Q\) returns
\[
    Y^\mu := \frac1M\sum_{j\in[M]} e_{B_j},
    \qquad
    B_j\simiid r_Q,
\]
which can be simulated using at most \(M\) i.i.d.\ samples from \(Q\). After escape, the conditional distribution is deterministic, and the oracle simply returns the corresponding point mass.

The oracle is unbiased since \(\bbE [Y^\pi] = r_P, \bbE [Y^\mu] = r_Q \). Moreover, for the relative variance, we have 
\[
    \bbE\chi^2(Y^\pi\|r_P)
    =
    \sum_{a:r_P(a)>0}\frac{\Var(Y^\pi(a))}{r_P(a)}
    =
    \frac1M\sum_{a:r_P(a)>0}(1-r_P(a))
    =
    \frac{|\supp(r_P)|-1}{M}
    \le
    \frac{K-1}{M}
    \le
    \sig^2 .
\]
A similar argument holds for $(Y^\mu, r_Q)$. Therefore the above simulation defines a valid noisy prefix distribution oracle with relative variance at most \(\sig^2\).

\begin{theorem} 
\label{thm:noisy_lb_general} 
Fix $m \ge K\ge 3$, $n\ge \Omega(\log \tfrac{1}{\eps})$, and $\eps\in(0,c_0)$. Any algorithm that solves Problem~\ref{prob:tv-est} with constant success probability under noisy prefix distribution oracle access with relative variance at most $\sig^2$, requires \(\Omega(n\cdot \frac{1 \lor \sig^2 }{\eps^2 \log K})\) oracle queries.
\end{theorem}
\begin{proof}
The \(\Omega(\frac n {\eps^2})\) term is already implied by
Theorem~\ref{thm:prefix-conditional-lower-bound}. It remains to prove the noise-dependent term. We work with the hard instances constructed above, and suppose that there is an algorithm \(\calA\) that estimates \(\TV{\pi_P,\mu_Q}\) to additive error \(\eps\) using at most \(B\) noisy prefix queries with constant success probability. We use \(\calA\) to build an ordinary i.i.d.\ sample estimator for \(\TV{P,Q}\).

To answer a query to the \(\pi_P\) oracle at an alive prefix, the simulator draws \(M\) samples from \(r_P=(1-q)e_0+qP\). Each such sample is generated by first drawing \(C_j\sim\operatorname{Bern}(q)\): if \(C_j=0\), the simulator outputs \(0\); if \(C_j=1\), it uses one fresh i.i.d.\ sample from \(P\) and outputs that label. Thus one noisy prefix query to \(\pi_P\) consumes \(\operatorname{Bin}(M,q)\) samples from \(P\), with expectation \(Mq\). The simulation for \(\mu_Q\) is identical, using samples from \(Q\). All escaped or off-path queries are answered deterministically and consume no samples.

Therefore, if \(\calA\) makes at most \(B\) oracle queries, the expected total number \(N\) of ordinary i.i.d.\ samples used from \(P\) and \(Q\) satisfies \(\bbE[N]\le BMq=BM/n\). By Markov's inequality, we may truncate the simulation after a sufficiently large constant multiple of \(BM/n\) samples while losing only a constant amount of success probability. Let \(\widehat D\) be the output of \(\calA\), and return \(\widehat d:=\widehat D/\alpha_n\). Since \(\TV{\pi_P,\mu_Q}=\alpha_n\TV{P,Q}\) and \(\alpha_n=\Theta(1)\), this gives an i.i.d.\ sample estimator for \(\TV{P,Q}\) to additive error \(O(\eps)\), using \(O(BM/n)\) total samples, with constant success probability. By Proposition~\ref{prop:dist_est_supp_k}, this is possible only if
\[
    \frac{BM}{n}
    =
    \Omega\paren{\frac{K}{\eps^2 \log K}} .
\]
Since \(M=O(K/\sig^2)\) , we obtain \(B=\Omega(n\sig^2/(\eps^2\log K))\). Combining this with the \(\Omega(n/\eps^2)\) lower bound, and weakening the latter to
\(\Omega(n/(\eps^2\log K))\), gives \(B=\Omega(n\cdot \frac{1\lor \sig^2}{\eps^2\log K}),\) as claimed.
\end{proof}
As a corollary of Theorem~\ref{thm:noisy_lb_general} and the oracle reduction in Lemma~\ref{lem:oracle_reduction}, we obtain a query complexity bound for Problem~\ref{prob:tv-est} under noisy prefix distribution  oracle access.

\begin{corollary}
\label{cor:sampler-lower}
Fix $m\ge 3$, $n\ge 2$, and $\eps\in(0,c_0)$. Any algorithm that solves Problem~\ref{prob:tv-est} with constant success probability under prefix sampler oracle access, over instances whose next-token supports have size at most $K$, requires \(\Omega(\frac{nK}{\eps^2 \log K})\) oracle queries.
\end{corollary}
\section{Case study: vllm vs.\ sglang}
\label{sec:case-study}

We return to the motivating question of whether we can measure how far the served distribution moves when the serving stack changes. In this section, we fix one serving configuration and measure the TV distance between the output distributions of \texttt{vllm} and \texttt{sglang} serving identical \texttt{Qwen3-0.6B} weights (bf16, top-$k = 20$) under that configuration (Appendix~\ref{app:exp-setup}).

Our estimators need each engine's log-probability for every sampled sequence. The engines' built-in scoring APIs compute these on a different code path from generation and are biased relative to the sampling distribution (Appendix~\ref{app:exp-replay} quantifies this bias). We instead \emph{replay} each sequence through the engine: at every step, a logit processor overrides the sampled token with the next token of the given sequence, and we record the log-probability the engine computed at that step.

Figure~\ref{fig:engine-tv} reports the result: $\TV{\pi,\mu} = 0.586$ at length $n = 500$ on the reference prompt. The measured distance does not depend solely on the engine pair: it grows with the sequence length, its \emph{support mismatch} component (trajectories containing a token to which one engine's top-$k$ assigns zero mass, on which an empirical KL is infinite) grows to about half of the estimate at $n = 500$, and it varies across prompts by an amount comparable to its own magnitude. The takeaway from this experiment is that cross-engine distances should be reported together with the length, truncation, and serving configuration under which they were measured.

\begin{figure}[t]
    \centering
    \includegraphics[width=\linewidth]{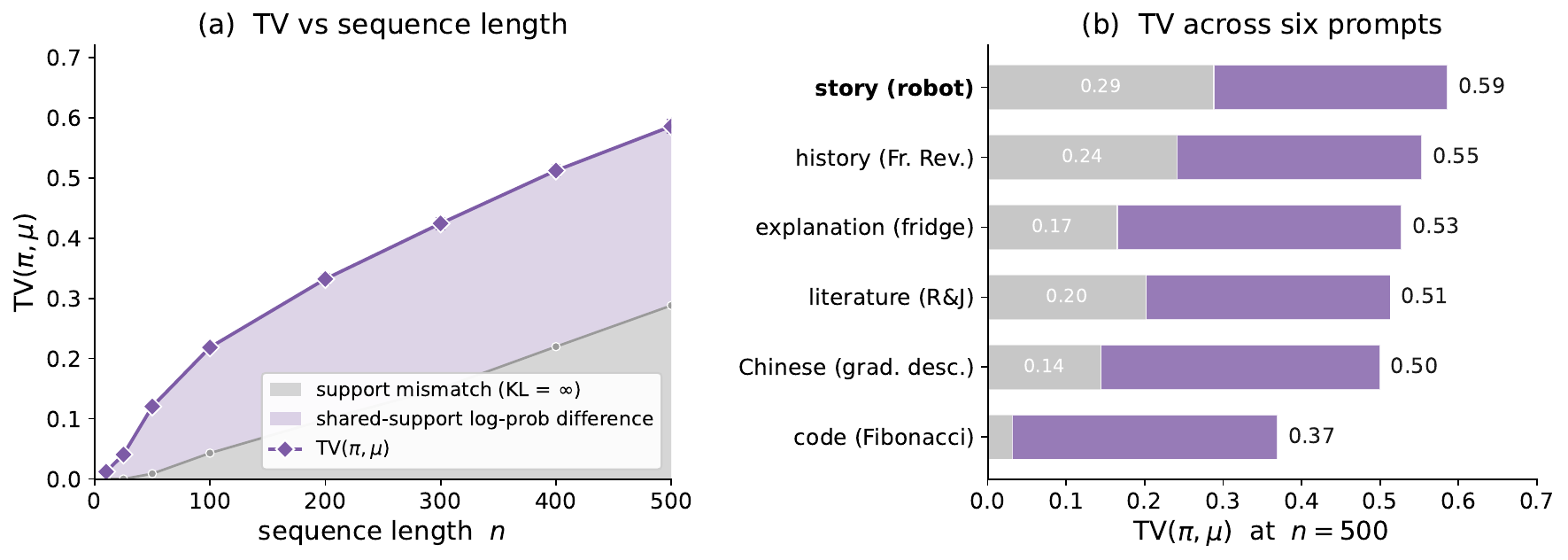}
    \caption{\textbf{The total variation distance between two inference engines.} \texttt{vllm} and \texttt{sglang} serving identical \texttt{Qwen3-0.6B} weights (bf16, top-$k=20$), measured at the sampling condition ($B = N = 512$, replay scoring; Appendix~\ref{app:exp-replay}). The estimate decomposes at the trajectory level into a \emph{shared-support log-probability difference} (both engines retain every sampled token and only the probabilities differ; purple) and \emph{support mismatch} (some sampled token receives zero mass from one engine's top-$k$, so the empirical KL diverges; gray). (a) The distance grows with sequence length, with the mismatch component reaching about half of the estimate at $n = 500$. (b) The same decomposition across six prompts; the distance varies with the token-confidence profile of the content (Appendix~\ref{app:exp-entropy}).}
    \label{fig:engine-tv}
\end{figure}

At the reference measurement, each engine generated its $N = 512$ sequences in a single batch; we call this the \emph{sampling condition}. Rescoring the same trajectories while changing only the scoring-time batch construction changes the estimate substantially: fixed scoring batches of $B = 256$ and $B = 1024$ give $0.363$ and $0.423$, compared with $0.586$ under the sampling condition (protocol and estimand in Appendix~\ref{app:exp-protocol}). The sensitivity to the serving condition is of the same order as the sensitivity to the prompt. The oracle noise depends on the condition as well: the per-call relative standard deviation $\sigma$ is ${\approx}3\times10^{-4}$ at the sampling condition, rises to $\sigma = 0.019$ for \texttt{vllm} at $B = 256$ (\texttt{sglang}'s replays there are bit-identical across repeats), and is at most $0.006$ for either engine at $B = 1024 \ge N$.

Two of the measurements above provide noisy instances on which to test the multilevel schedule of Algorithm~\ref{alg:blackbox-tv}: the $B = 256$ rescoring oracle, and the long-sequence measurement at $n = 2000$ (Appendix~\ref{app:exp-length}), for which the required number of repeated queries per position, $r \gtrsim n\sigma^2$, is substantial even at small per-call noise. The best schedule depends on the oracle noise, so we choose it from the data: the first $64$ trajectories serve as a pilot, from which we estimate the variance at each level and select the schedule; the pilot's cost is included in every budget we report (protocol in Appendix~\ref{app:exp-mlmc}). Figure~\ref{fig:mlmc-real}(a) then compares budgets at equal accuracy: to reach a given accuracy, the pilot-chosen schedule needs up to ${\approx}1.3\times$ less budget than the best single-level estimator, i.e., a single fixed repetition count $r$ chosen with hindsight separately for each accuracy. Figure~\ref{fig:mlmc-real}(b) shows that the method is robust to the choice of schedule: whatever the inner levels, the multilevel estimator is unbiased for the value defined by its top level, which sets the bias floor exactly as the fixed $r$ of a single-level estimator does. Schedules sharing the same top level therefore converge to the same value and differ only in how quickly.

\begin{figure}[t]
    \centering
    \includegraphics[width=\linewidth]{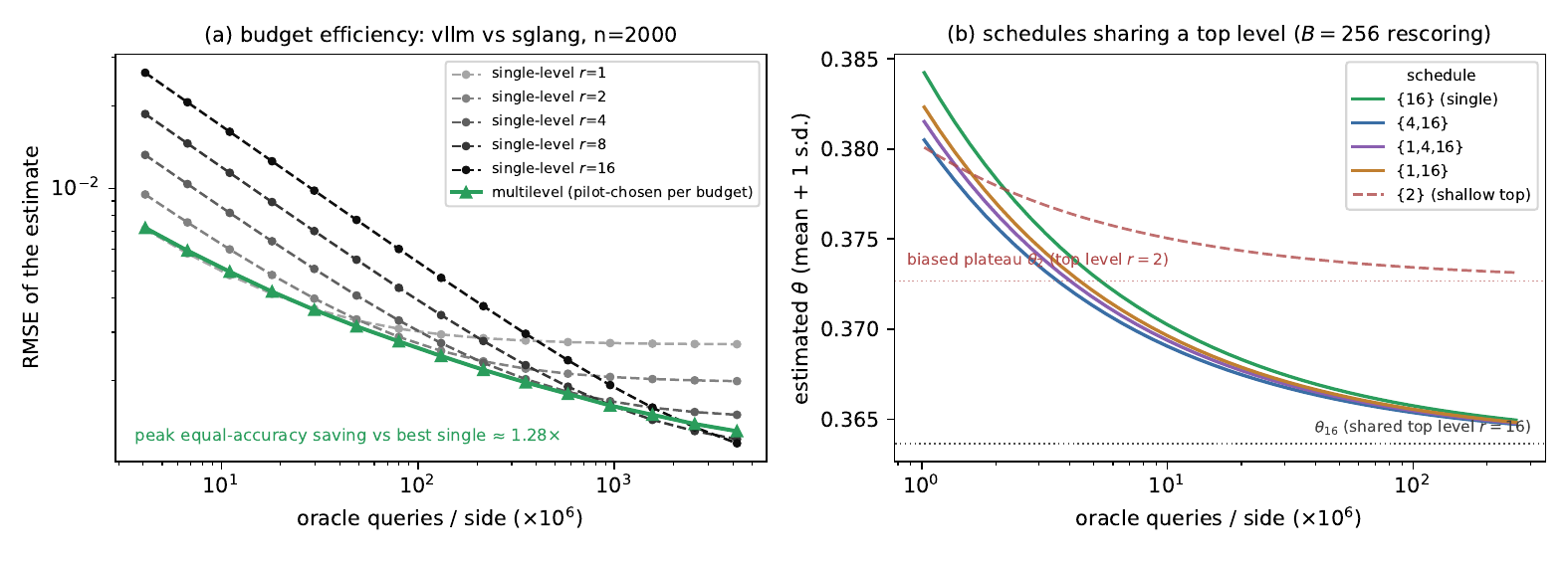}
    \caption{\textbf{Multilevel variance reduction on two noisy conditions from the case study.} (a) Engine pair, $n = 2000$: dashed curves are fixed single-level configurations (repetition count $r$), each flattening at its bias floor; the green curve is the pilot-chosen schedule at each budget, all pilot costs included (protocol in Appendix~\ref{app:exp-mlmc}). (b) Engine pair, $B = 256$ rescoring: schedules sharing the top level $r = 16$ converge to the same answer, with inner levels setting only the speed of convergence; a shallow top level ($r = 2$, dashed) converges to a different, biased value.}
    \label{fig:mlmc-real}
\end{figure}

\section{Conclusion}
Measuring the KL divergence between LLMs has become a basic tool in LLM research, while the TV distance, the most natural alternative to KL, remains underexplored. The KL and TV have complementary strengths: the former is unbounded and sensitive to changes in the tail, whereas the latter can saturate quickly for large distribution changes. Motivated by the strengths of TV as an estimable quantity from finite samples, in this paper, we studied how to efficiently measure the TV distance between the output distributions of two LLMs.

To capture the situation when measuring real-world deployments, we propose three levels for modeling the query access given by LLMs: \emph{sample-only access, exact logit access, and noisy logit access}, the last of which reflects the nondeterminism of LLM inference. We designed algorithms that are more sample-efficient than prior work under all three access models, proved tightness of our exact logit access upper bound, and validated the practicality of our estimators experimentally.

Our experiments show that, in settings like distance estimation where the object of interest is the distribution change itself, the noisy-logit model is both faithful and necessary. We thus believe that our formulation of this model holds independent interest beyond TV estimation. For instance, the existing literature on measuring KL divergence generally treats the returned logits as exact, without accounting for nondeterminism.

Finally, we highlight some considerations for practitioners using our estimators. 
\begin{enumerate}
    \item The measured value is sensitive to the measurement condition. For the same model weights, serving choices such as the kernels or the batching shift the output distribution itself (Figure~\ref{fig:engine-tv}). Distances should therefore be reported with stated measurement condition.
    \item Scoring the same next-token prediction repeatedly denoises the measurement. Without denoising, the estimated TV distance is systematically inflated. One caveat of such denoising is that repeated scores are not perfectly independent, since each evaluation can alter the system state seen by the next. In our experimental setup, this correlation is small enough that the repeats are approximately independent, and repeated scoring achieves its purpose.
    \item All scoring must follow the generation path exactly. A sequence-scoring API, for example, invokes batching strategies and other execution details that differ from those of generation. Consequently, it produces scores different from the distribution the engine generates from.
    \item The magnitude of the oracle noise differs across inference engines. A small pilot sample is enough to learn the noise of the engine at hand: on the pilot, given the target accuracy of the final measurement, one can directly select the best denoising level or multilevel schedule, and then use it when scoring at full scale.
\end{enumerate}
\section*{Acknowledgments}
This work is supported by the NSF AI Institute for Foundations of Machine Learning (IFML).

\newpage
\bibliography{ref}
\bibliographystyle{alpha}
\newpage
\appendix
\section{Comparison between TV and KL}
\label{sec:TVKL}

The comparison between TV and KL has a long history in probability, statistics, and information theory \cite{cover1999elements,gibbs2002choosing,sason2016f,canonne2022short}. Here we sketch the aspects most relevant to our setting. We first recall several standard comparisons between total variation distance and KL divergence. A classical relation is Pinsker's inequality,
\[
    \TV{P,Q} \le \sqrt{\frac12 \KL{P}{Q}} ,
\]
which shows that small KL divergence implies small TV distance. However, this
relation is only one-sided and can be loose in several regimes. For example,
two distributions may be arbitrarily close in TV while having infinite KL
divergence if their supports do not match. Concretely, considering $P_\eps = \Bern(\eps), Q = \Bern(0)$, we have
\[
    \TV{P_\varepsilon,Q}=\varepsilon,
    \qquad
    \KL{P_\varepsilon}{Q}=+\infty .
\]
This illustrates that KL is highly sensitive to support mismatch, whereas TV
only charges the amount of probability mass on which the two distributions
differ. Conversely, when two distributions are already very far apart, TV
saturates at its maximum value \(1\), while KL, being unbounded, can still
reflect different degrees of separation.

The two quantities also differ in their operational meanings. TV admits the optimal coupling characterization
\[
    \TV{P,Q}
    =
    \inf_{\gamma\in\Gamma(P,Q)}
    \Pr_{(X,Y)\sim\gamma}[X\neq Y],
\]
where \(\Gamma(P,Q)\) denotes the set of all couplings of \(P\) and \(Q\).
Equivalently, there exists a maximal coupling such that
\[
    \Pr[X=Y]=1-\TV{P,Q}.
\]
Thus, a statement such as \(\TV{P,Q}=0.01\) has a clear sample-level meaning:
the two distributions can be coupled so that their samples are exactly
identical with probability \(0.99\). This makes TV particularly natural for
measuring observable output-level discrepancy between two generative
distributions.

TV also has a direct decision-theoretic interpretation. Consider the binary
Bayesian testing problem where one first samples a hidden label
\(H\in\{0,1\}\) uniformly at random, and then observes \(X\sim P\) if
\(H=0\) and \(X\sim Q\) if \(H=1\). The minimum possible probability of error
over all decision rules is
\[
    P_{\mathrm{err}}^\star
    =
    \frac{1-\TV{P,Q}}{2}.
\]
Thus, TV directly measures the distinguishability of two distributions.In other words, if two APIs are treated as black-box generators and we want to
detect whether their output distributions differ, TV gives the fundamental
statistical limit of distinguishability. A small TV implies that no black-box
detector can reliably tell the two APIs apart from a single output sample,
whereas a large TV implies that the two APIs are distinguishable with small
Bayesian error.

Compared with TV, KL divergence and its variants are more widely used in
empirical machine learning. For example, the cross-entropy loss used to train
language models satisfies
\[
    \mathbb E_{X\sim P}[-\log Q_\theta(X)]
    =
    H(P)+\KL{P}{Q_\theta},
\]
and perplexity is an exponential transformation of this average log-loss. Thus,
standard language-model training can be viewed as minimizing a KL-type
objective.

This popularity is partly due to two practical advantages. First, KL admits an
exact chain rule under autoregressive factorization:
\[
    \KL{P}{Q}
    =
    \sum_{t=1}^n
    \mathbb E_{X_{<t}\sim P}
    \left[
        \KL{
            P_t(\cdot\mid X_{<t})
        }{
            Q_t(\cdot\mid X_{<t})
        }
    \right].
\]
Thus, a sequence-level KL decomposes into token-level terms that are directly
accessible from prefix probabilities. Second, KL-based objectives are smooth
and naturally compatible with gradient-based training. In contrast, TV is
bounded and operationally meaningful, but it lacks an analogous chain rule and
is less convenient as a direct differentiable loss.


\section{Experimental details and additional results}
\label{app:experiments}

This appendix records the experimental setup behind Figures~\ref{fig:oracle-sigma}--\ref{fig:mlmc-real}: hardware and software (Appendix~\ref{app:exp-setup}), the synthetic validation of Algorithm~\ref{alg:blackbox-tv} (Appendix~\ref{app:exp-synthetic}), the calibration of the noise parameter $\sig$ (Appendix~\ref{app:exp-sigma-cal}), the empirical top-$k$ support size (Appendix~\ref{app:exp-topk-union}), and the construction and validation of the replay scoring oracle (Appendix~\ref{app:exp-replay}). It also gives the measurement and evaluation protocols of the case study (Appendices~\ref{app:exp-protocol} and~\ref{app:exp-mlmc}) and four additional measurements referenced from the main text: the dependence of the cross-engine distance on the sequence length (Appendix~\ref{app:exp-length}), its dependence on the top-$k$ truncation (Appendix~\ref{app:exp-topk-dependence}), its dependence on the prompt (Appendix~\ref{app:exp-entropy}), and the empirical cost of sample-only access (Appendix~\ref{app:exp-access}).

\subsection{Setup}
\label{app:exp-setup}

\paragraph{Hardware and software.}
All GPU experiments ran on single NVIDIA GH200 nodes (96\,GB HBM3e) of a shared SLURM-managed cluster, with toolchain \texttt{gcc 14.2.0} / \texttt{cuda 12.8} / \texttt{python 3.11.8}. The engine experiments use \texttt{vllm 0.19.1} (\texttt{enforce\_eager=True}, prefix caching disabled) with \texttt{torch 2.10.0+cu128} and \texttt{transformers 5.7.0}, and \texttt{sglang 0.5.10.post1} (CUDA graph capture disabled) with \texttt{torch 2.9.1+cu128} and \texttt{transformers 5.3.0}, in isolated virtual environments; graph capture and prefix caching are disabled so that determinism properties of the measurement reflect the model computation rather than request-level caching. The attention-backend experiments of Figures~\ref{fig:oracle-sigma} and~\ref{fig:denoise} use HuggingFace Transformers (\texttt{torch 2.10.0+cu128}, \texttt{transformers 4.57.6}) with the PyTorch SDPA backend pinned to cuDNN, FlashAttention-2, or math, on \texttt{Qwen3-0.6B}, \texttt{Qwen3-1.7B}, \texttt{Qwen3-8B}, \texttt{Qwen2.5-7B}, and \texttt{Qwen3-30B-A3B} (a 30B mixture-of-experts model), all in bf16. Synthetic experiments run on CPU with NumPy, master seed $10$.

\paragraph{Synthetic hard instance.}
We use the binary autoregressive lower-bound construction of Section~\ref{ssec:lower_whit_box}: each sequence factors as $x = u \oplus z$ with $u \in \Sigma^r$ a block index and $z \in \Sigma^{n-r}$ a within-block leaf, and within each active block the discrepancy between $\pi$ and $\mu$ is concentrated at the final bit. Throughout the synthetic experiments we fix $n = 128$, $r = 12$ (the block-index length of Section~\ref{ssec:lower_whit_box}; not a repetition count), $p_{\rm active} = 0.4$, $\alpha = 0.49$, giving $\TV{\pi, \mu} = 0.388$ by direct enumeration. The noisy oracle of Definition~\ref{def:conditional-sketch} is realized by adding zero-mean Gaussian noise of standard deviation $\sig\sqrt{\hat p (1 - \hat p)}$ to each realized next-token probability (deterministic transitions remain deterministic, and the average over repeated queries converges to the true probability); at $\sig = 0$ this reduces to the exact prefix logit oracle of Definition~\ref{def:next_token_logprob}. The perturbation is applied to the realized-token probability only (a scalar sketch of the vector oracle), and individual draws are left unclipped (they may leave $[0,1]$), so the construction satisfies the unbiasedness and relative-variance conditions of Definition~\ref{def:conditional-sketch} exactly while relaxing simplex membership; the only guard is that the $r$-averaged probability is clipped to $[10^{-12}, 1]$ before the logarithm. The sample oracle is realized as the one-hot empirical frequency of resampled tokens.

\paragraph{Cross-engine sampling protocol.}
The shared input is the system prompt \texttt{"You are a helpful assistant."} followed by a user prompt, tokenized by the Qwen3 chat template to a $30$-token prefix. Each side samples $N = 512$ continuations of length $n = 500$ in a single \texttt{generate} call ($B = N = 512$; top-$k = 20$, temperature $1$, EOS suppressed for fixed-length sampling); the $\pi$ side (\texttt{vllm}) uses base seed $0$ and the $\mu$ side (\texttt{sglang}) base seed $1$, with an independent per-call seed offset for each repeated scoring pass. Sampling and scoring for one engine take place within a single engine session, since reloading kernels between passes can produce spuriously byte-identical repeats. The reference prompt is \texttt{"Tell me a story about a robot learning to paint."}; the cross-prompt experiment of Figure~\ref{fig:engine-tv}(b) repeats the identical protocol on five further prompts (an explanation, a history question, a literature summary, a Chinese-language explanation, and a coding task). We store $R = 16$ scoring passes per side for the reference measurement, $R = 64$ for the $n = 2000$ run, and $R = 128$ and $R = 32$ for the $B = 256$ and $B = 1024$ rescoring conditions of Appendix~\ref{app:exp-protocol}.

\paragraph{Repeat-axis hygiene.}
Serving stacks batch queries internally, so consecutive scoring passes can share scheduler state and be serially correlated along the repeat axis: on the stored pools we measure short-lag autocorrelations of up to $\pm 0.08$ (\texttt{vllm}) and slow drift (\texttt{sglang}). All statistics computed along the repeat axis (variances, half-splits, level increments) are therefore evaluated after an independent uniform permutation of the repeats at each (trajectory, position) cell, and multilevel quantities are averaged over $20$ such permutations. The permutation does not make dependent repeats independent: it makes them exchangeable, so that at the level of second moments the residual dependence is a common per-cell component, which cancels to leading order in the difference-based statistics used here (the half-split calibration of Appendix~\ref{app:exp-sigma-cal} and the coupled level increments). The diagnostics are sensitive to this step: evaluating the same quantities in raw serving order inflates the variance ratios $\rho_\ell$ by roughly $2\times$ and the peak saving of Figure~\ref{fig:mlmc-real}(a) from ${\approx}1.3\times$ to ${\approx}1.8\times$. The permutation is therefore part of the protocol; one dataset-specific exception is noted in Appendix~\ref{app:exp-mlmc}.

\subsection{Synthetic validation}
\label{app:exp-synthetic}

\paragraph{Exact-logit sanity check.}
At $\sig = 0$, Algorithm~\ref{alg:tv-test-white-box} reduces to a standard outer-trajectory Monte Carlo on the bounded estimator $R = (1 - \exp(L_\mu - L_\pi))_+$, with $\E[R] = \TV{\pi, \mu}$ by Lemma~\ref{lem:tv-representations}. Figure~\ref{fig:exact-logit-convergence} plots the empirical mean absolute error against $1/\sqrt{N}$ and overlays the Gaussian-asymptotic reference $\sqrt{2/\pi}\,\sqrt{\Var(R)/N}$. The fitted slope is $-0.493$ against the theoretical $-0.500$, and the empirical curve sits slightly below the reference because $R \in [0, 1]$ has shorter-than-Gaussian tails. This confirms that the synthetic instance, the true-TV computation, and the estimator are mutually consistent before any noise is added.

\begin{figure}[t]
    \centering
    \includegraphics[width=0.6\textwidth]{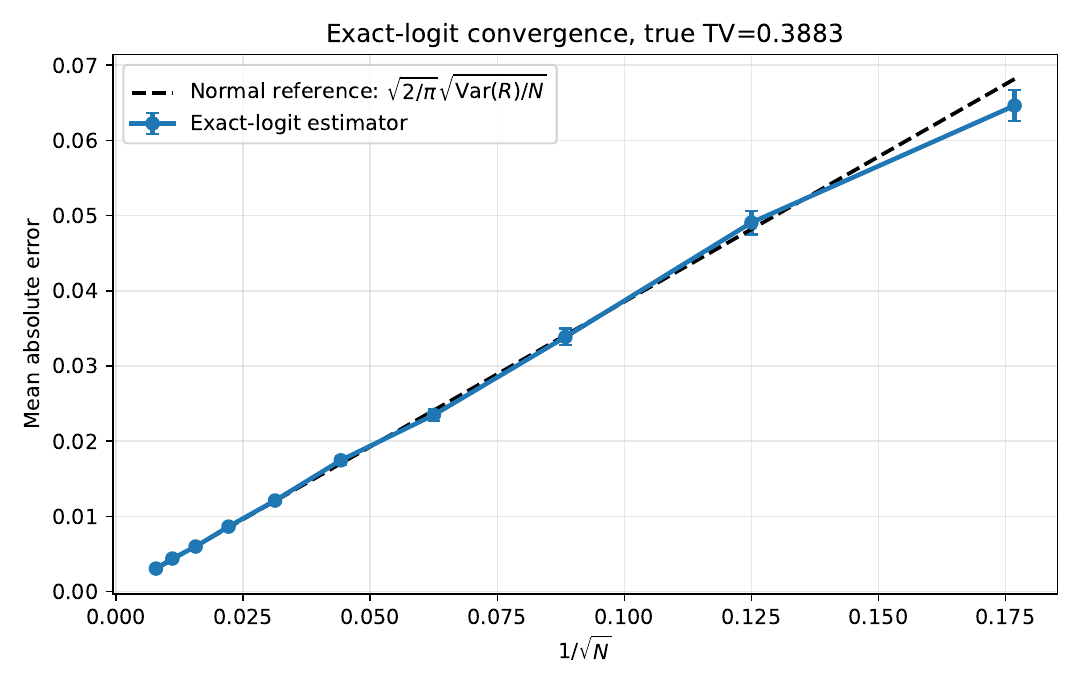}
    \caption{Mean absolute error of Algorithm~\ref{alg:tv-test-white-box} on the synthetic hard instance at $\sig = 0$ (exact logits), against $1/\sqrt{N}$, with $\TV{\pi, \mu} = 0.388$ known by enumeration. The dashed line is the Gaussian-asymptotic reference $\sqrt{2/\pi}\,\sqrt{\Var(R)/N}$. The fitted slope is $-0.493$ against the theoretical $-0.500$.}
    \label{fig:exact-logit-convergence}
\end{figure}

\paragraph{Multilevel depth and the variance-ratio diagnostic.}
On the same instance with noise level $\sig > 0$, the benefit of including each additional level of Algorithm~\ref{alg:blackbox-tv} is captured by the empirical variance ratio
\begin{equation*}
    \rho_\ell \defeq \frac{\Var\bra{Y_\ell}}{\Var\bra{Z_\ell}}
        = \frac{\Var\paren{Z_\ell - Z_{\ell-1}}}{\Var\paren{Z_\ell}},
\end{equation*}
where $Y_\ell$ and $Z_\ell$ are the coupled level-$\ell$ increment and estimator defined in \eqref{eq:hzdef}--\eqref{eq:yl_def}. When $\rho_\ell \ll 1$, the cheaper level $\ell - 1$ provides an essentially free variance reduction at level $\ell$; when $\rho_\ell$ is of order one, level $\ell - 1$ adds cost without reducing variance and should be dropped. Figure~\ref{fig:synth-rho-regime} sweeps $\sig \in \{0.04, 0.5, 0.8\}$ across four schedule depths under matched query budget, with $\rho_\ell$ for each level reported in the column titles. All schedules converge to the shared $r = 256$ target at every $\sig$ (its residual bias sets the error floor visible at $\sig = 0.5, 0.8$), and the empirical winner tracks $\rho_\ell$ throughout: at $\sig = 0.04$ all $\rho_\ell < 0.02$ and the deepest ($4$-level) schedule matches the accuracy of the single-level estimator at $r = 256$ with roughly $7\times$ less budget; as $\sig$ grows, $\rho_1$ rises past $1$ and the optimal depth shrinks. Since $\rho_\ell$ is measurable on a small calibration sample, the depth can be chosen before the estimation budget is spent; the pilot protocol of Appendix~\ref{app:exp-mlmc} operationalizes this on real data.

\begin{figure}[t]
    \centering
    \includegraphics[width=\linewidth]{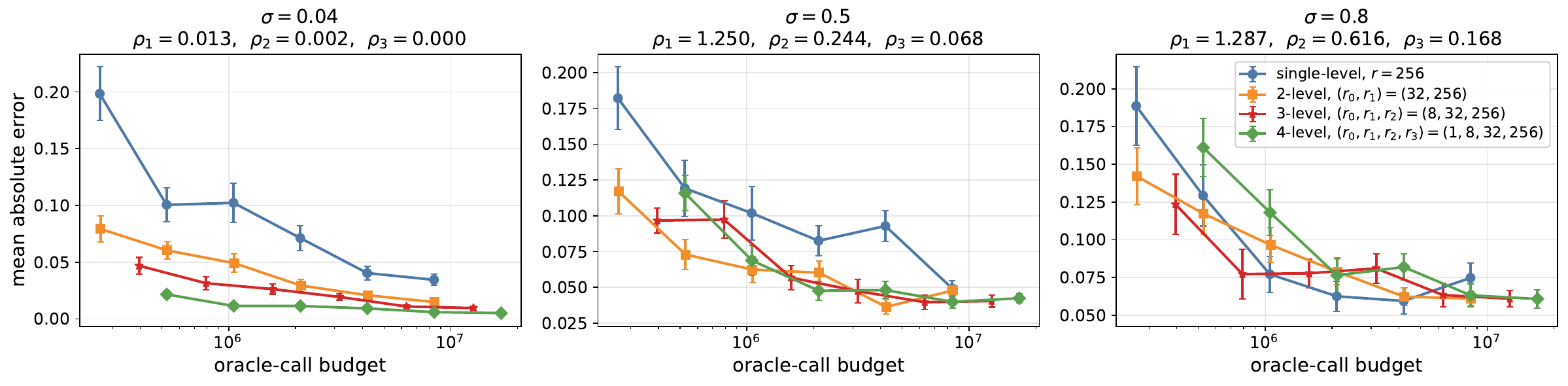}
    \caption{Mean absolute error against oracle-call budget for $1$-, $2$-, $3$-, and $4$-level instantiations of Algorithm~\ref{alg:blackbox-tv} on the synthetic hard instance, at three noise levels $\sig \in \{0.04, 0.5, 0.8\}$. Column titles report the variance ratio $\rho_\ell$ at each level of the deepest schedule.}
    \label{fig:synth-rho-regime}
\end{figure}

\subsection{Calibration of the noise parameter $\sig$}
\label{app:exp-sigma-cal}

The parameter $\sig^2$ of Definition~\ref{def:conditional-sketch} is the relative variance of a single oracle output against the \emph{true} prefix-conditional distribution $p$. On a real model, $p$ is unavailable, so we substitute an empirical proxy and correct for its residual noise. On a cache of $R_{\max}$ repeated queries at each of a large sample of (trajectory, position) cells, we split the repeats into two disjoint halves $A$ and $B$ of size $R_{\rm half}$, let $\hat p_{r,A}$ be the average of the first $r$ repeats of $A$ and $\hat p_B$ the full average over $B$ (both unbiased, and approximately independent; see Appendix~\ref{app:exp-setup}), and use the fact that, by Lemma~\ref{lem:empirical-tree-chi2}, to leading order
\[
    \mathbb{E}\Brack{\chi^2(\hat p_{r,A} \,\|\, \hat p_B)}
    \;\approx\;
    \sig^2 \paren{\frac{1}{r} + \frac{1}{R_{\rm half}}}.
\]
The $\chi^2$ statistic is evaluated on the union of the top-$k$ sets across repeats, so it includes probability mass that moves across the truncation boundary between calls. Measuring the left-hand side, averaging over cells, and dividing by the prefactor yields one $\sig^2$ estimate per averaging depth $r$; if the noisy-oracle model holds, these estimates must coincide. On \texttt{Qwen3-0.6B} with $R_{\max} = 248$ ($100 \times 500$ cells, $r \in \{1, 2, 4, 8, 16, 32, 64, 124\}$) the eight estimates agree within $\pm 15\%$, indicating that $\sig^2$ is a property of the serving configuration rather than an artifact of the chosen depth. The five-model measurement of Figure~\ref{fig:oracle-sigma}(a,c) applies the same estimator with $r = 1$ on $1{,}500$ subsampled cells per configuration, under the two batching conditions defined there: \emph{isolated} (the query alone in a fixed batch, on a warmed-up model) and \emph{served} (the query batched with other, varying concurrent requests within a serving session).

Two remarks on the functional. First, the union-support $\chi^2$ is the conservative choice: a variant restricted to the probability of the realized token only, which measures the noise the estimator experiences on shared support, is smaller (by about $3\times$ at $k = 20$ on \texttt{Qwen3-0.6B}); the theorems consume the conservative value. Second, single-call deviations are heavy-tailed (excess kurtosis ${\approx}9$; Figure~\ref{fig:oracle-sigma}(b)), so variance-type aggregation across cells is essential; median-type aggregates measure the width of the central peak and underestimate $\sig^2$ by roughly $4\times$. The distribution across cells is also heavy-tailed: at $k = 20$, the 99th percentile of the per-cell estimate is roughly $12\times$ the mean and the largest observed cell is roughly $100\times$, so a worst-case per-prefix bound is correspondingly larger than the averages reported here.

\subsection{Empirical top-$k$ support size}
\label{app:exp-topk-union}

The bounds in the introduction are stated in terms of the support size $K$ of the next-token distribution. When the API exposes a stated top-$k$ parameter, the relevant $K$ is the union of per-call top-$k$ sets across the $r$ repeated queries used at a fixed prefix; nondeterminism in which tokens reach the truncation threshold can make this union strictly larger than $k$. Figure~\ref{fig:topk-union} measures the union size at $k = 20$ on all five models over $r$ up to $64$: the union exceeds the stated $k$ by $3$--$5\%$ on the dense models and by ${\sim}10\%$ on the MoE (whose expert routing perturbs the truncation boundary), and endpoints for $k = 5$ and $k = 100$ lie in the same $3$--$10\%$ range. The per-cell tails are moderate as well: the 99th-percentile union stays below $1.2k$ on the dense models and $1.35k$ on the MoE, and the largest union observed in any cell is $1.75k$. The practical $K$ is therefore comparable to the API-exposed $k$ and far below $2k$.

\begin{figure}[t]
    \centering
    \includegraphics[width=0.62\textwidth]{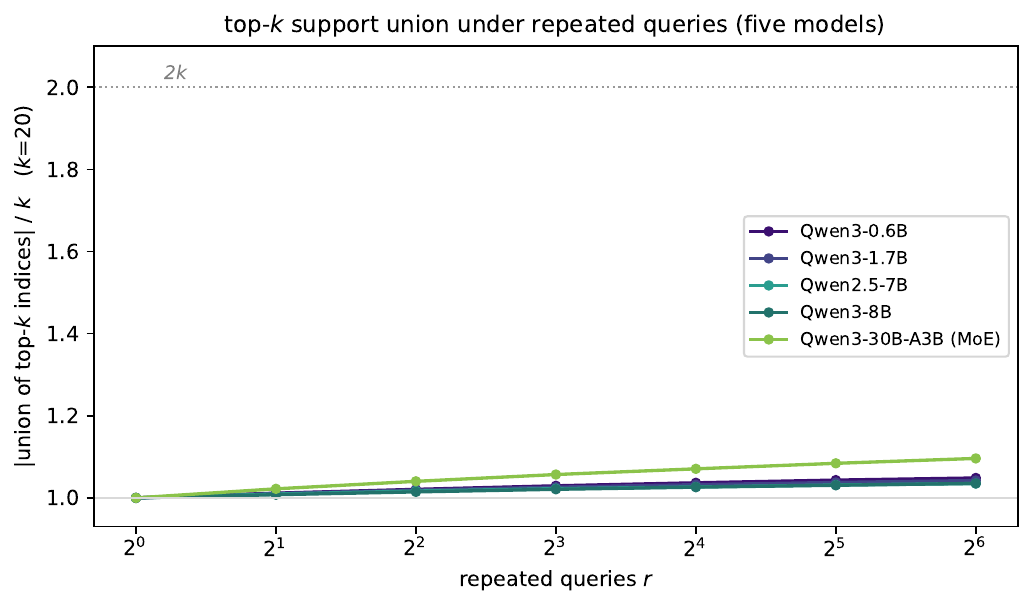}
    \caption{Mean empirical top-$k$ support union $|\bigcup_{i \le r} \mathrm{top}\text{-}k(i)| / k$ against the number of repeated queries $r$ at a fixed prefix ($k = 20$; $1{,}500$ (trajectory, position) cells per model). Across five models the union exceeds the stated $k$ by at most ${\sim}10\%$, far below $2k$.}
    \label{fig:topk-union}
\end{figure}

\subsection{The replay scoring oracle}
\label{app:exp-replay}

\paragraph{Construction.}
Given a production engine and a sampled trajectory $X$, the case study requires an oracle whose output is distributed as the engine's \emph{own} next-token distribution at every prefix of $X$. We obtain one in three steps. \emph{(i)} Trajectories are sampled natively as in Appendix~\ref{app:exp-setup}. \emph{(ii)} A per-request logit processor forces each sampling step to the corresponding token of $X$ (teacher forcing), so the engine decodes along $X$ through its ordinary decode path; the processor records the step's top-$k$ log-softmax \emph{before} the forcing is applied, since log-probabilities returned by the engine are computed after the processor and are contaminated by it. \emph{(iii)} The captured top-$k$ log-probabilities are renormalized over the top-$k$ set, giving one unbiased draw of the next-token distribution; $r$ independent replays give $r$ approximately i.i.d.\ draws (Appendix~\ref{app:exp-setup} measures the serial correlation across passes).
Because the replayed computation matches generation position-by-position (same decode kernels, same batch shape), the replay oracle scores the same distribution the engine samples from.

\paragraph{Validation.}
Three checks confirm the construction. The forcing is exact: the rate at which an engine assigns zero mass to its own sampled token is $0$ under replay scoring. The captured values agree with the log-probabilities observed at sampling time to ${\sim}4 \times 10^{-5}$. Repeated replays under the sampling-time batch shape agree to a per-call relative standard deviation of $\sigma \approx 3\times10^{-4}$, more than three orders of magnitude below the distances we report.

\paragraph{Sequence-level scoring APIs are biased.}
The engines' sequence-level log-probability API scores an entire given sequence in one call and is the interface a practitioner would reach for first. It is not faithful to the sampling distribution: its execution path and tie-breaking differ from those of generation. On the same trajectories, it assigns zero mass to $37\%$ of \texttt{sglang}'s own sampled trajectories ($0.6\%$ for \texttt{vllm}; both rates are $0$ under replay scoring), and it systematically compresses the shared-support log-probability differences; the net effect moves the estimated TV from $0.586$ to $0.550$ (Figure~\ref{fig:seq-api}). This failure is a bias rather than noise (its repeat variance is ${\sim}10^{-8}$ and does not accumulate along the sequence), so black-box TV measurements must verify that the scoring path matches the generation path, which the replay construction provides.

\begin{figure}[t]
    \centering
    \includegraphics[width=0.56\textwidth]{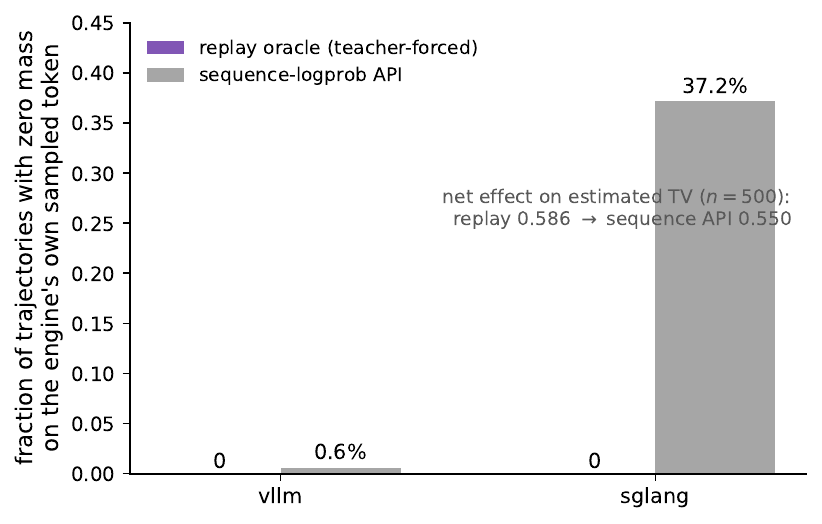}
    \caption{Self-scoring faithfulness of two scoring interfaces. Because oracle noise at this condition is small, a faithful scoring oracle should almost never assign zero mass to a token the engine itself just sampled. Under replay scoring the rate is $0$ for both engines; under the sequence-level log-probability API it is $0.6\%$ for \texttt{vllm} and $37\%$ for \texttt{sglang}, a deterministic bias (repeat variance ${\sim}10^{-8}$) caused by the mismatch between the scoring and generation paths. The net effect moves the estimated cross-engine TV from $0.586$ (replay) to $0.550$ (sequence API).}
    \label{fig:seq-api}
\end{figure}

\subsection{Case-study measurement protocol}
\label{app:exp-protocol}

\paragraph{Estimator decomposition.}
The plug-in estimator $\hZ$ of \eqref{eq:hzdef} splits at the trajectory level into the two components shown in Figures~\ref{fig:engine-tv} and~\ref{fig:topk-sweep}. If both engines assign positive $r$-averaged mass to every sampled token of the trajectory, then $\hZ = \tanh(|\log \hpi_r - \log \hmu_r|/2)$ (the \emph{shared-support log-probability difference}); if at least one position receives zero mass from one engine's $r$-averaged top-$k$, then $\hZ = 1$ exactly (the \emph{support mismatch} component, on which the empirical KL divergence is $+\infty$). At the reference condition, the mismatch component accounts for $0.29$ of the $0.586$ estimate at $n = 500$; its growth with sequence length is shown in Appendix~\ref{app:exp-length}.

\paragraph{Symmetric scoring.}
The mixture representation gives $\E_{X \sim \pi}[\hZ] + \E_{X \sim \mu}[\hZ] = 2\,\TV{\pi,\mu}$ in the exact limit, so averaging the two one-sided expectations cancels the second-order bias that either incurs alone. All reported distances use this symmetric average, with one exception noted in Table~\ref{tab:mlmc-landscape}; on the reference data the two one-sided estimates differ from their mean by at most $\pm 0.011$ (${\sim}2\%$), and the residual effect of the averaging depth ($r = 8$ against a held-out deep reference) is ${\sim}0.002$.

\paragraph{Serving-condition rescoring.}
The rescoring experiments of Section~\ref{sec:case-study} keep the trajectories and the engines fixed and change only the scoring-time batch construction: the $N = 512$ targets are partitioned into fixed batches of size $B \in \{256, 1024\}$ (padding sequences complete the batch at $B = 1024$), with the batch composition independently resampled at every repeat, and both scoring-side truncations $k \in \{20, 100\}$ stored; the reported rescoring values and the multilevel evaluation of Appendix~\ref{app:exp-mlmc} use the $k = 100$ readout. Because the trajectories remain those of the sampling condition, the estimand of a rescoring dataset is $\tfrac{1}{2}\,\E_{X \sim \pi}[\hZ] + \tfrac{1}{2}\,\E_{X \sim \mu}[\hZ]$ with the oracles evaluated at batch size $B$; it is not $\TV{\pi_B, \mu_B}$, which would require resampling the trajectories under each condition. Holding the trajectory measure fixed makes the comparison a controlled one, in which only the scoring oracle changes. The resulting oracle noise is a property of the condition: at $B = 256$, \texttt{vllm}'s scheduler is active and its oracle carries $\sigma = 0.019$ while \texttt{sglang}'s repeats are bit-identical; at $B = 1024 \ge N$ the noise is small for both engines ($\sigma \le 0.006$).

\subsection{Multilevel evaluation protocol}
\label{app:exp-mlmc}

Figure~\ref{fig:mlmc-real} evaluates the multilevel schedule of Algorithm~\ref{alg:blackbox-tv} under a protocol in which every quantity the method needs is estimated from the data it would see. The first $P = 64$ trajectories serve as a pilot whose queries count against the budget; the pilot estimates per-level variances and biases, selects among all schedules of at most four levels on a fixed grid of repetition counts, and allocates samples across the selected levels proportionally to $\sqrt{V_\ell / r_\ell}$. Reported curves average over $40$ independent pilot draws and $20$ repeat-axis permutations (Appendix~\ref{app:exp-setup}). RMSE is computed analytically per configuration: the bias of each level is measured against a deep reference built from a held-out half of the stored repeats (floored at its paired standard error), and variances are taken from the full data, so the pilot affects selection only; budget points beyond the stored cache ($6.6 \times 10^{7}$ queries per side) extrapolate the $1/N$ variance scaling around the measured bias floors. The pilot's queries are recycled into the final estimate; a replay-based control ($300$ simulations per budget) shows that recycling introduces no detectable selection bias (both arms carry the same plug-in bias of the selected level) while improving RMSE by ${\approx}45\%$ at small budgets.

Two baselines are reported. The \emph{pilot-chosen single level} is given the same pilot information and must likewise pick one repetition count; the \emph{hindsight-best single-level family} is the pointwise lower envelope of all fixed single-level curves (the dashed family in Figure~\ref{fig:mlmc-real}(a)), which no realizable procedure can attain. The headline saving of ${\approx}1.3\times$ in the main text is measured against the latter, stronger baseline; against the pilot-chosen single level the peak saving on the same dataset is ${\sim}4.4\times$. Table~\ref{tab:mlmc-landscape} summarizes the three case-study datasets under this protocol. The $n = 2000$ dataset evaluates the schedule on the $\pi$-side arm of the estimator: the full estimator is the average of the two one-sided arms, and the multilevel schedule operates on the repeat axis within each arm, so arm-level comparisons carry over to the full estimator. The degenerate row is informative: at $B = 1024$ both oracles are nearly deterministic, so there is no noise for a multilevel schedule to average and the pilot correctly degrades to a single level (this row is evaluated in serving order: its residual deviations are rare, serially clustered serving-state switches rather than per-query noise, and a repeat-exchangeable permutation would misread them as averageable). A mismatch-dominated distance degenerates for a structural reason: a trajectory locked at $\hZ = 1$ carries no repeat variance, so repetition and multilevel schedules have nothing to average, and at tight truncations, where mismatch dominates the distance (Appendix~\ref{app:exp-topk-dependence}), multilevel offers no gain.

\begin{table}[t]
\centering
\begin{tabular}{lccc}
\toprule
\textbf{Dataset (condition)} & $\theta$ & \textbf{RMSE ratio} & \textbf{Peak saving} \\
\midrule
Engine pair, $n = 2000$ (sampling condition) & $0.68$ (one-sided) & $0.80$ & $4.4\times$ \\
Engine pair, $B = 256$ rescoring & $0.363$ & $0.91$ & $1.5\times$ \\
Engine pair, $B = 1024$ rescoring & $0.423$ & $1.00$ & --- \\
\bottomrule
\end{tabular}
\vspace{.4em}
\caption{Multilevel evaluation on the case-study datasets. $\theta$ is the estimand of the dataset (a TV only for the sampling-condition rows; Appendix~\ref{app:exp-protocol}); the RMSE ratio is the minimum, over budgets, of the pilot-chosen multilevel RMSE divided by the pilot-chosen single-level RMSE (values below $1$ favor multilevel); the peak saving is the equal-accuracy budget ratio against the pilot-chosen single level (the same baseline as the RMSE ratio; Figure~\ref{fig:mlmc-real}(a) additionally reports the saving against the stronger hindsight-best family). The $n = 2000$ dataset is scored one-sided; the identity of Appendix~\ref{app:exp-protocol} bounds the deviation from the symmetric value. Multilevel helps where the oracle carries averageable noise and degrades to a single level where the oracle is deterministic ($B = 1024$).}
\label{tab:mlmc-landscape}
\end{table}

\subsection{Dependence on the sequence length}
\label{app:exp-length}

A separate campaign repeats the reference configuration with declared length $n = 2000$. It gives $0.541$ at $n = 500$, against the reference value $0.586$, and $0.677$ at $n = 2000$, where the support-mismatch component covers $48\%$ of trajectories.

The gap to the reference value traces to \texttt{sglang}'s admission policy: it reserves KV cache at the declared length, and $512 \times 2000$ tokens do not fit one GH200 ($112$\,KiB per token, ${\approx}111$\,GB), so its running batch is smaller from the first decode step and differs again between the sampling and scoring passes. Its rate of assigning zero mass to its own sampled tokens, $0$ in the reference campaign, is here $19\%$ of trajectories at $n = 500$ and $52\%$ at $n = 2000$. \texttt{vllm} is unaffected: its trajectories agree bit-for-bit with the reference campaign over the first $500$ tokens.

The $n = 2000$ values are therefore specific to this campaign, and part of the mismatch component reflects the scoring-environment split rather than the engine difference. The multilevel evaluation of Table~\ref{tab:mlmc-landscape} is unaffected, since its estimand is defined by the campaign's own oracles.

\subsection{Dependence on the top-$k$ truncation}
\label{app:exp-topk-dependence}

The cross-engine distance is a property of the served distribution at a stated API truncation. Resampling and rescoring the reference configuration at top-$k \in \{5, 20, 100\}$ (the same truncation used for sampling and scoring) moves the estimate from $0.86$ to $0.59$ to $0.45$ at $n = 500$: the support-mismatch component collapses from $0.73$ to $0.06$ as the truncation loosens, while the shared-support component grows from $0.13$ to $0.39$ (Figure~\ref{fig:topk-sweep}). A substantial part of the distance at tight truncations therefore consists of probability mass moving across the truncation boundary; since the three measurements do not exhibit a plateau, we do not extrapolate a truncation-free limit, and we report cross-engine distances at a stated top-$k$; larger values of $k$ reduce the truncation-sensitive component.

\begin{figure}[t]
    \centering
    \includegraphics[width=0.5\textwidth]{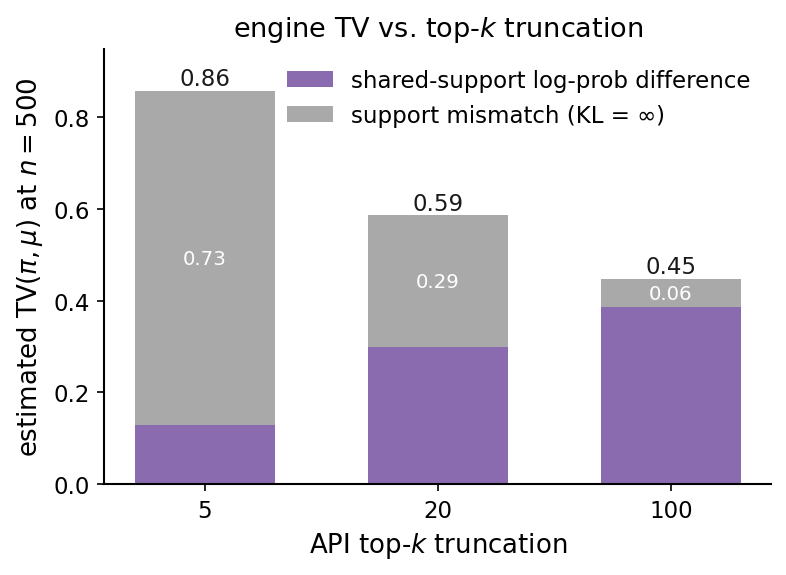}
    \caption{The cross-engine TV depends strongly on the API top-$k$ truncation ($n = 500$; sampling and scoring use the same truncation). As $k$ grows from $5$ to $100$ the estimate falls from $0.86$ to $0.45$: the support-mismatch component (gray; empirical $\mathrm{KL} = \infty$) collapses from $0.73$ to $0.06$, while the shared-support log-probability difference (purple) grows from $0.13$ to $0.39$.}
    \label{fig:topk-sweep}
\end{figure}

\subsection{Prompt dependence and token confidence}
\label{app:exp-entropy}

Figure~\ref{fig:engine-tv}(b) shows the cross-engine distance varying from $0.37$ (code) to $0.59$ (story) across prompts. A comparison of the two extremes suggests that this variation is largely attributable to the token-confidence profile of the content rather than to prompt-specific engine behavior (Figure~\ref{fig:entropy-mechanism}). The per-token log-probability differences of the two prompts belong to the same distribution family, the code prompt being only ${\sim}21\%$ narrower ($\sigma_\Delta = 0.040$ against $0.051$); the code prompt concentrates its sampled-token probabilities near $1$ (mean $0.70$ against $0.54$); and after binning per-token differences by the token probability $\hat p$, both prompts fall on a single curve $\sigma(\Delta \mid \hat p) \approx 0.09\,(1 - \hat p)$. On this pair of prompts, per-token engine disagreement is a function of token confidence alone, with the prompt entering only through its confidence profile: high-confidence tokens both disagree less and sit far from the truncation boundary (making support mismatch rare), so high-confidence content is smaller in both components. We report this as a suggestive observation from the two extreme prompts rather than a validated mechanism.

\begin{figure}[t]
    \centering
    \includegraphics[width=\linewidth]{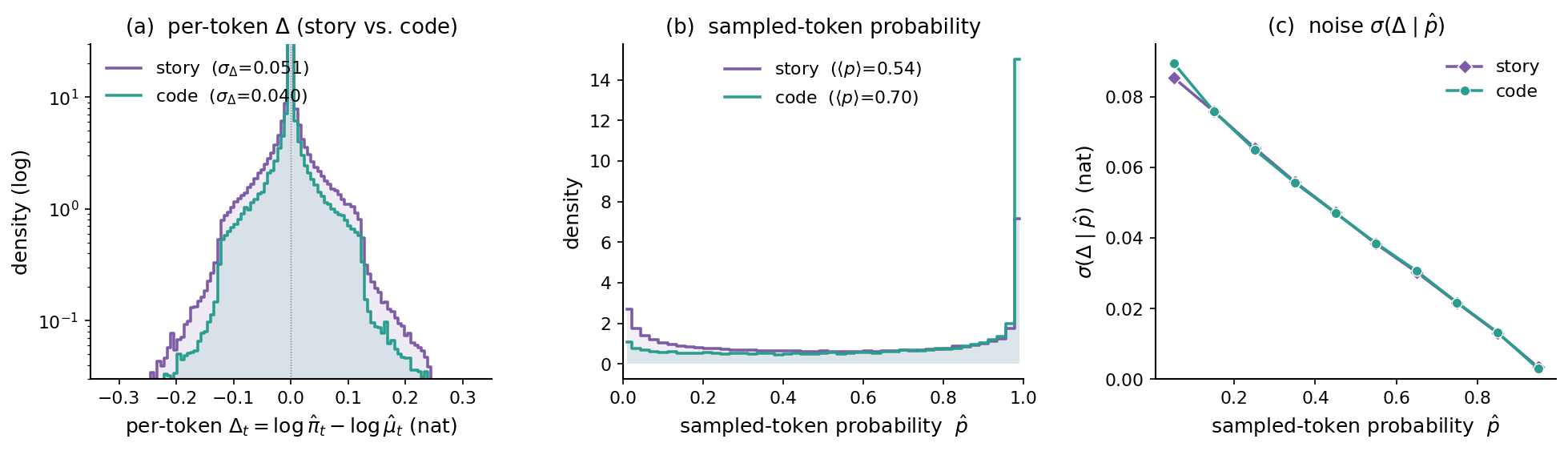}
    \caption{Why the cross-engine distance varies with the prompt (story vs.\ code; positions where both engines retain the sampled token). \emph{(a)} The per-token log-probability differences $\Delta_t$ of the two prompts belong to the same distribution family; code is ${\sim}21\%$ narrower ($\sigma_\Delta = 0.040$ vs.\ $0.051$). \emph{(b)} Sampled-token probabilities: the code prompt concentrates near $\hat p \approx 1$ (mean $0.70$ vs.\ $0.54$). \emph{(c)} Binning $\Delta_t$ by token probability, both prompts follow one curve $\sigma(\Delta \mid \hat p) \approx 0.09\,(1-\hat p)$: on this pair of prompts, per-token disagreement is a function of token confidence alone, with the prompt entering only through its confidence profile.}
    \label{fig:entropy-mechanism}
\end{figure}

\subsection{The cost of sample-only access}
\label{app:exp-access}

The access hierarchy of the introduction has a direct empirical counterpart. On the synthetic instance, running the same estimator with the three oracles of Definitions~\ref{def:next_token_logprob}--\ref{def:conditional-sketch} shows exact-logit access converging at the clean Monte Carlo rate (fitted slope $-0.50$), sample access converging markedly more slowly (slope $-0.32$), and noisy-logit access interpolating between them as $\sig^2$ varies; at RMSE $0.05$, sample access requires roughly $600\times$ more queries than logit access, and the gap widens with the accuracy requirement (Figure~\ref{fig:access-gap}). On the real engine data the cost is more extreme: the median per-sequence minimum realized-token probability is ${\sim}8\times10^{-4}$, so resolving every token from samples requires ${\gtrsim}10^3$ resamples per position, and a naive fixed-$r$ product estimator recovers only $0.16$ of the replay-oracle value $0.586$ even at $2.5\times10^{5}$ resamples per sequence. When the distributional difference lies partly on low-probability tokens, logit access is necessary in practice.

\begin{figure}[t]
    \centering
    \includegraphics[width=0.62\textwidth]{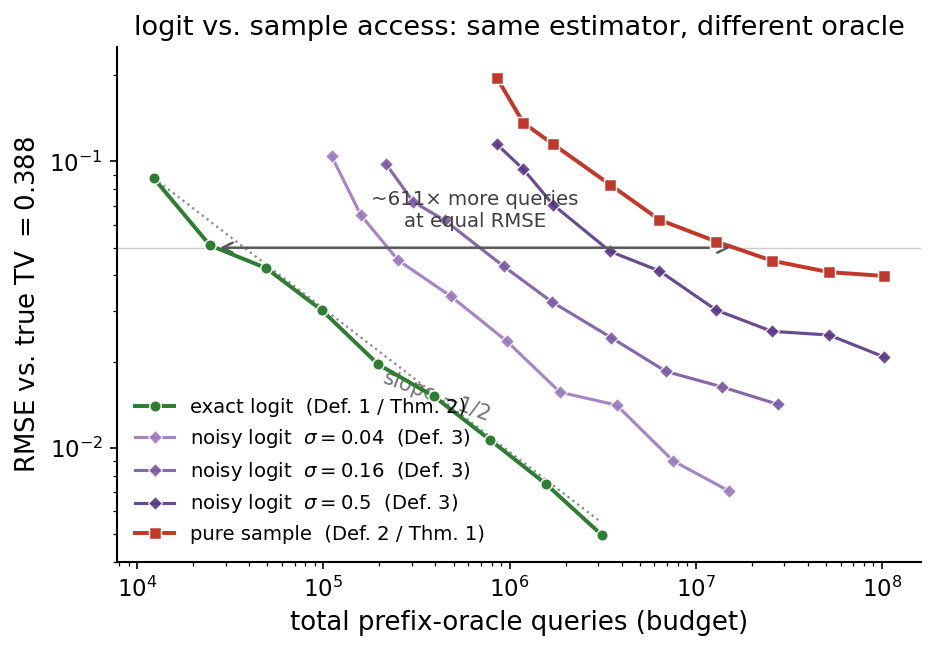}
    \caption{Logit versus sample access: the same estimator run with different oracles on the synthetic hard instance (true $\mathrm{TV} = 0.388$ by enumeration; $64$ repetitions per point). Exact-logit access (Definition~\ref{def:next_token_logprob}) converges at the Monte Carlo rate (slope $-0.50$); sample access (Definition~\ref{def:next_token_sample}) converges more slowly (slope $-0.32$); noisy-logit access (Definition~\ref{def:conditional-sketch}, $\sig \in \{0.04, 0.16, 0.5\}$) interpolates between them. At equal RMSE $0.05$, sample access requires ${\approx}600\times$ more queries.}
    \label{fig:access-gap}
\end{figure}

\end{document}